\documentclass{article} 
\usepackage[OT1]{fontenc}
\PassOptionsToPackage{table}{xcolor}
\usepackage{iclr2027_conference,times}

\usepackage{amsmath,amsfonts,bm}

\def\eqref#1{equation~\ref{#1}}

\def\1{\bm{1}}

\DeclareMathAlphabet{\mathsfit}{\encodingdefault}{\sfdefault}{m}{sl}
\SetMathAlphabet{\mathsfit}{bold}{\encodingdefault}{\sfdefault}{bx}{n}

\newcommand{\KL}{D_{\mathrm{KL}}}

\usepackage{amsmath,amssymb,amsthm}
\usepackage{enumitem}
\usepackage{hyperref}
\usepackage{url}
\usepackage{booktabs}
\usepackage{multirow}
\usepackage{graphicx}
\usepackage[table]{xcolor}

\title{%
\raisebox{-0.75em}{\includegraphics[height=2.5em]{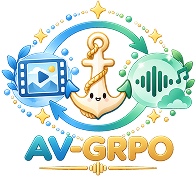}}%
\,AV-GRPO: Modality-Anchored Decoupling Diffusion Reinforcement Learning for Joint Audio-Video Generation%
}

\author{%
\parbox[t]{\dimexpr\textwidth-2\tabcolsep\relax}{%
\raggedright
\textbf{Zhiyu Xu}\textsuperscript{1,2}\quad
\textbf{Weilong Yan}\textsuperscript{3}\quad
\textbf{Yufei Shi}\textsuperscript{4}\quad
\textbf{Shiyang Li}\textsuperscript{5}\quad
\textbf{Yihao Liu}\textsuperscript{1}\\[2pt]
\textbf{Kin-Man Lam}\textsuperscript{2,*}\quad
\textbf{Yuewen Cao}\textsuperscript{1,*}\\[4pt]
\normalfont\small
\textsuperscript{1}Shanghai AI Laboratory\quad
\textsuperscript{2}The Hong Kong Polytechnic University\\[3pt]
\textsuperscript{3}National University of Singapore\quad
\textsuperscript{4}Nanyang Technological University\\[3pt]
\textsuperscript{5}Zhejiang University\quad
\textsuperscript{*}Corresponding authors. 
}}

\iclrfinalcopy 
\begin{document}

\maketitle

\begin{abstract}
Recent years have witnessed remarkable progress in joint audio-video generation. Nevertheless, existing models still suffer from limited per-modality fidelity, inadequate text-modality alignment, and weak cross-modal synchronization. Reinforcement-learning-based post-training offers a promising avenue for addressing these shortcomings. However, naively extending such approaches to joint audio-video generation is challenging. Heterogeneous multimodal rewards entangle the learning signals of the two modalities and obscure credit assignment. Jointly optimizing both modality towers also incurs substantial computational cost despite their distinct optimization dynamics. Furthermore, the difficulty of synchronization evaluation varies with the sampled counter-part modality, making fair reward comparison difficult. To address these challenges, we propose AV-GRPO, a modality-anchored online diffusion reinforcement learning framework, together with 5DAV, a fully decoupled and difficulty-controllable training dataset. AV-GRPO integrates three key components: (1) modality-anchored rollouts that disentangle learning signals while reducing anchor-induced difficulty variation; 
(2) trajectory-locked frozen-tower optimization that reduces training costs and redirects credit assignment, thereby simplifying optimization.
and (3) adaptive objectives and perturbation strengths tailored to modality-specific dynamics and sample difficulty. Collectively, these designs transform coupled multi-modal preference learning into a set of conditional unimodal subproblems, enabling more accurate reward attribution and more effective synchronization optimization. We construct 5DAV, a dataset decoupled along five dimensions, to facilitate systematic training. Experiments on JavisBench and VABench demonstrate that AV-GRPO consistently outperforms LTX-2.3 in generation quality, semantic alignment, and cross-modal synchronization under both LoRA and full fine-tuning settings. Extensive ablation studies further validate the effectiveness of the proposed designs. Our code and data is available at \href{https://github.com/zhiyuxu03/AV-GRPO}{AV-GRPO}. 
\end{abstract}

\section{Introduction}

Joint audio--video generation has advanced rapidly through unified and interacting dual-stream models~\citep{hacohen2026ltx,low2025ovi,liu2025javisdit,team2026mova}. Yet generated audio and video still struggle to achieve high quality together: either stream may contain perceptual artifacts, one or both may not reflect the text prompt, and otherwise plausible streams may depict mismatched events or drift out of sync. Scaling pretraining alone does not directly prioritize these failures. Reward-guided post-training instead turns perceptual and semantic evaluators into learning signals, with demonstrated benefits in language and visual generation~\citep{rafailov2023direct,shao2024deepseekmath,liu2026flow}. Extending it to joint generation requires improving both streams while preserving their semantic and temporal dependence.

A natural baseline treats each generated audio--video pair as one policy output. For every prompt, it samples groups of paired rollouts, evaluates audio quality and alignment, video quality and alignment, and cross-modal synchronization, aggregates the heterogeneous rewards into a group-relative advantage, and updates both towers~\citep{shao2024deepseekmath,liu2026flow,zheng2026diffusionnft,xue2025dancegrpo}. This straightforward joint optimization has three difficulties.

\textbf{First, mixed rewards complicate model optimization and fair reward comparison.} Audio quality, visual quality, text alignment, and synchronization may favor different candidates, making their joint optimization difficult. 
Moreover, synchronization rewards depend on the sampled counterpart: matching audio to a steady scene can be easier than frequent visual events. A higher score may therefore reflect an easier counterpart as well as better alignment, making comparisons across jointly varying pairs less controlled. 
\textbf{Second, joint updates incur high memory costs and can misassign credit.} Backpropagation and training states are required for both towers, yet a shared advantage updates both even when its improvement primarily comes from one modality. Their interactions throughout denoising further complicate assigning that improvement to the responsible tower. 
\textbf{Third, different modalities have different optimization dynamics.} Their pretrained capabilities, latent scales, reward sensitivities, and learning speeds differ. Shared objectives and noise strengths can favor one branch or destabilize the other, making a common configuration unsuitable.

We propose \textbf{AV-GRPO}, which alternates audio-anchored video optimization and video-anchored audio optimization (Figure~\ref{fig:Figure_pipeline}). 
\textbf{First, modality-anchored rollouts disentangle rewards and control comparison conditions.} Each group shares one complete anchor trajectory while independently sampling the target modality. The anchor reward can be omitted, simplifying optimization to target-modality quality and alignment plus synchronization against a common counterpart. This controls anchor-induced difficulty variation within each group, while refreshing anchors across groups preserves diverse conditions. 
\textbf{Second, trajectory locking and tower freezing significantly reduce memory costs and direct credit to the target tower.} The anchor states remain fixed throughout each rollout and its policy update, and only the target tower is optimized. Applying the conditional advantage to this tower avoids updating both branches with the same signal; freezing the counterpart also reduces backward-pass and training-state costs.
\textbf{Third, hyperparameter decoupling accommodates different branch dynamics.} We use modality-specific reward compositions, noise strengths, and loss weightings to support exploration and stable optimization in each tower. Alternating the target modality allows both towers to improve under controlled cross-modal conditions.

To support controlled post-training, we introduce \textbf{5DAV}, a training set of 5,760 prompts organized along five independently specified dimensions. We evaluate AV-GRPO on JavisBench and VABench against LTX-2.3 and GDPO under both LoRA and full fine-tuning. The results show broad gains in perceptual quality, text alignment, cross-modal coherence, and synchronization, while ablations test the dataset and the alternating schedule. The frozen-tower strategy also enables full-parameter post-training of the 22B LTX-2.3 model on eight NVIDIA A800 GPUs. 
Our contributions are:
\begin{itemize}
    \item We propose \textbf{AV-GRPO}, a modality-anchored reinforcement learning framework that alternates controlled modality-wise updates for clearer reward attribution and balanced audio--video improvement.
    \item We develop three complementary mechanisms: (i) modality-anchored rollouts that disentangle reward objectives and enable controlled synchronization comparisons; (ii) trajectory locking and tower freezing that isolate modality-wise credit assignment while reducing memory overhead; and (iii) decoupled optimization hyperparameters that accommodate asymmetric audio--video learning dynamics.
    \item We introduce \textbf{5DAV}, a five-dimensionally decoupled training set. Experiments on JavisBench and VABench show broad gains in generation quality, text--modality alignment, and synchronization, supported by ablations.
\end{itemize}

\section{Method}
\subsection{Preliminaries}
\label{sec:prelim}
\begin{figure}[t]
    \centering
    \includegraphics[width=1.0\linewidth]{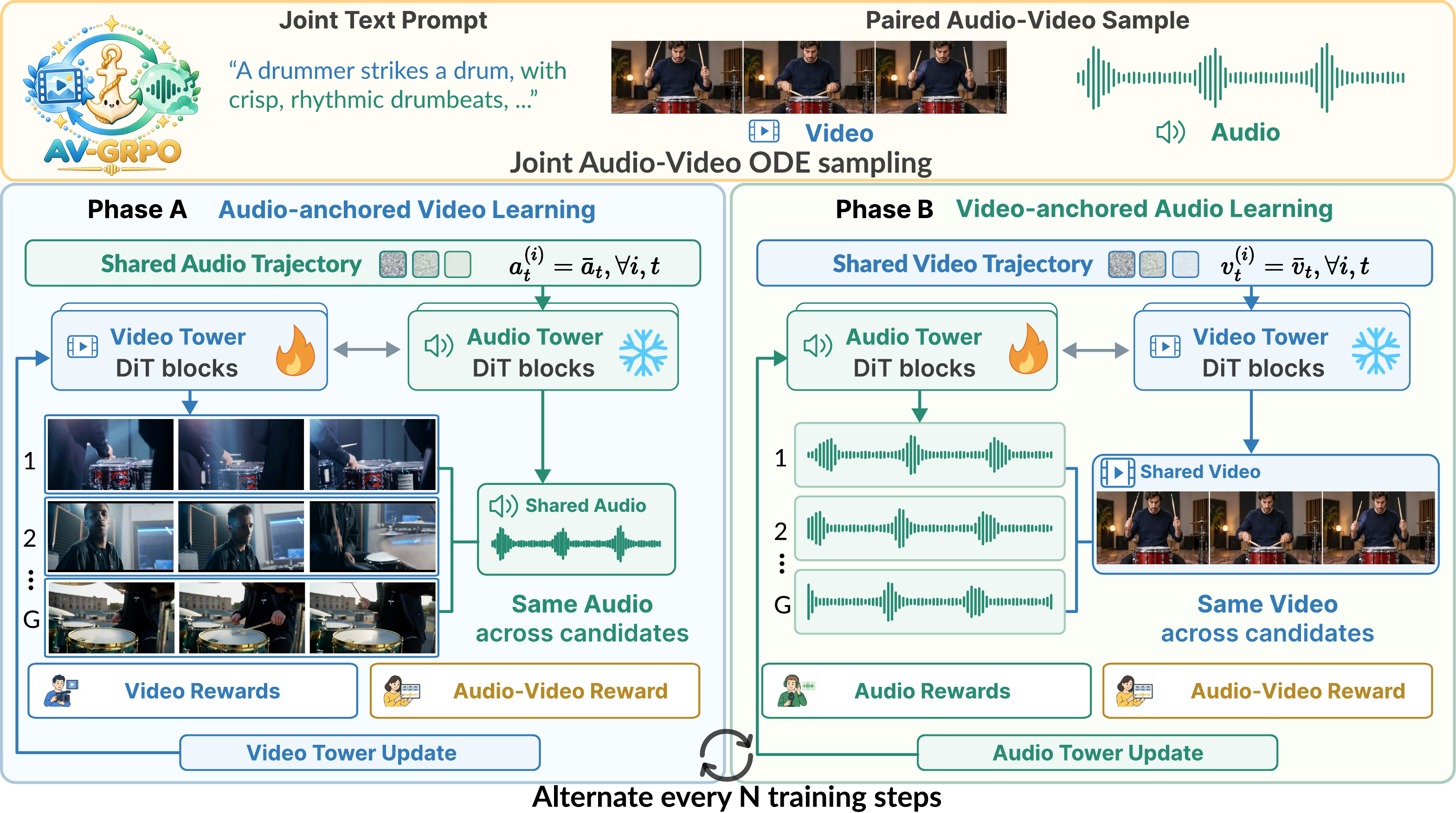}
    \vspace{-5mm}
    \caption{Overview of the AV-GRPO training pipeline.}
    \label{fig:Figure_pipeline}
    \vspace{-5mm}
\end{figure}

\textbf{Joint audio--video generation models} produce a video and its audio track from a text prompt $c$ within a single denoising process. \citep{hacohen2026ltx,liu2025javisdit}

The prevailing approach is to represent video and audio in separate latent spaces, denoted $x^{\mathrm{V}}$ and $x^{\mathrm{A}}$, and noise them following rectified flow,
\begin{equation}
x^{m}_t=(1-t)\,x^{m}_0+t\,\epsilon^{m},\qquad \epsilon^{m}\sim\mathcal N(0,I),\quad m\in\{\mathrm{A},\mathrm{V}\},
\label{eq:interp}
\end{equation}
where the timestep $t$ is shared by the two modalities and the noises are independent. The denoiser has one stream per modality, which we call the two towers and whose parameters we write as $\theta=(\theta_{\mathrm{A}},\theta_{\mathrm{V}})$; the towers exchange information through cross-modal attention or other mechanisms, so the velocity predicted by either tower depends on the current latents of both modalities:
\begin{equation}
\mathrm{d}x^{m}_t=v^{m}_\theta\big(x^{\mathrm{A}}_t,x^{\mathrm{V}}_t,t,c\big)\,\mathrm{d}t,\qquad m\in\{\mathrm{A},\mathrm{V}\}.
\label{eq:ode}
\end{equation}
At sampling time, both modalities integrate Eq.~(\ref{eq:ode}) in parallel from $t=1$ to $t=0$.

\textbf{Flow-GRPO.}
Online reinforcement learning with GRPO~\citep{shao2024deepseekmath} optimizes the step-by-step sampler as a policy: each denoising step is an action, and the reward is assigned to the final sample. This requires the sampling process to be stochastic, so that different samples can be drawn for the same prompt, and the per-step transition probability to be computable, so that there is a probability ratio to adjust. The deterministic ODE of Eq.~(\ref{eq:ode}) satisfies neither: the same initial noise always yields the same sample, 
and computing its transition probability is computationally expensive due to divergence estimation.
Flow-GRPO~\citep{liu2026flow}  therefore replaces the ODE with an SDE that has the same marginal distribution at every timestep, introducing stochasticity without changing the distribution the pretrained model generates. Discretized over $T$ steps, the sample advances by
\begin{equation}
x_{t+\Delta t}=x_t+\Big[v_\theta+\frac{\sigma_t^2}{2t}\big(x_t+(1-t)\,v_\theta\big)\Big]\Delta t+\sigma_t\sqrt{|\Delta t|}\;\epsilon,
\qquad \sigma_t=a\sqrt{\tfrac{t}{1-t}},
\label{eq:sde-step}
\end{equation}
where $v_\theta$ abbreviates $v_\theta(x_t,t,c)$, $\epsilon\sim\mathcal N(0,I)$ controls how strongly the sampling is perturbed. Each step of Eq.~(\ref{eq:sde-step}) is a Gaussian transition $\pi_\theta(x_{t+\Delta t}\mid x_t,c)$. With this sampler, Flow-GRPO draws $G$ trajectories per prompt and maximizes
\begin{equation}
J(\theta)=\mathbb E\Big[\frac1G\sum_{i=1}^{G}\frac1T\sum_{t}\Big(\min\big(r^{i}_t\hat A^{i},\ \mathrm{clip}(r^{i}_t,1-\varepsilon,1+\varepsilon)\hat A^{i}\big)-\beta\,D_{\mathrm{KL}}\big(\pi_\theta\,\|\,\pi_{\mathrm{ref}}\big)\Big)\Big],
\label{eq:grpo}
\end{equation}
where $\hat A^{i}=(R^{i}-\mathrm{mean}_i R^{i})/\mathrm{std}_i R^{i}$ is the advantage of trajectory $i$, obtained by standardizing the reward of its final sample, $R^{i}=R(x^{i}_0,c)$, within the group, and indicates whether the sample is better or worse than the others in its group; $r^{i}_t$ is the probability ratio between the current policy and the sampling policy at this step, i.e., a ratio of two Gaussian densities; $\varepsilon$ is the clip range; $\beta$ is the KL coefficient and $\pi_{\mathrm{ref}}$ the pretrained model, for which the KL term has a closed form proportional to $\|v_\theta-v_{\mathrm{ref}}\|^2$.

\subsection{AV-GRPO}
AV-GRPO applies the Flow-GRPO update of Eq.~(\ref{eq:grpo}) to one tower at a time. In an \emph{audio-anchored} phase, it fixes one audio trajectory, samples a group of video trajectories, and updates only the video tower; the \emph{video-anchored} phase is symmetric (Figure~\ref{fig:Figure_pipeline}). The phases alternate every $n$ training steps. Let $m\in\{\mathrm A,\mathrm V\}$ denote the anchor modality, $\bar m$ the target, and $\tau^m=(x^m_1,\ldots,x^m_0)$ a complete denoising trajectory.

\textbf{Modality-Anchored Rollouts.}
For prompt $c$, we first generate one audio--video pair by sampling both towers with Eq.~(\ref{eq:sde-step}). We retain the full trajectory $\tau^m_{\mathrm{anc}}$ of modality $m$ and independently sample $G$ target trajectories. At each step, its anchor state $x^m_{\mathrm{anc},t}$ replaces the state of modality $m$ in cross-modal attention. The $i$-th target advances by
\begin{equation}
x^{\bar m,i}_{t+\Delta t}=x^{\bar m,i}_{t}+\Big[v^{\bar m}_\theta+\frac{\sigma_t^2}{2t}\big(x^{\bar m,i}_{t}+(1-t)\,v^{\bar m}_\theta\big)\Big]\Delta t+\sigma_t\sqrt{|\Delta t|}\;\epsilon^{\bar m,i},
\qquad \epsilon^{\bar m,i}\sim\mathcal N(0,I),
\label{eq:anchored-step}
\end{equation}
where $v^{\bar m}_\theta$ is evaluated at $(x^{\bar m,i}_t,x^m_{\mathrm{anc},t},t,c)$. This gives a conditional Gaussian transition $\pi^{\bar m}_\theta(x^{\bar m}_{t+\Delta t}\mid x^{\bar m}_t,x^m_{\mathrm{anc},t},c)$ and $G$ final pairs $(x^m_{\mathrm{anc},0},x^{\bar m,i}_0)$. The anchor is common to the group, whereas the target sample varies.

\textbf{Reward Disentanglement.}
Each pair receives a target-modality quality and text-alignment reward $R_{\bar m}$ and a synchronization reward $R_{\mathrm{sync}}$ (Section~\ref{sec:rewards}). In a joint group, changes in either modality can alter the synchronization score. Here, the anchor reward is constant across all candidates and disappears under group centering, leaving
\begin{equation}
R^i=R_{\bar m}(x^{\bar m,i}_0)+R_{\mathrm{sync}}(x^{\bar m,i}_0,x^m_{\mathrm{anc},0}).
\label{eq:anchored-reward}
\end{equation}

Meanwhile, this controls the variation in synchronization difficulty of the anchor modality: within a group, all video (or audio) samples share the same audio (or the same video), enabling fair comparison and more targeted learning of audio–video synchronization.

This grouping changes the conditional policy being optimized, rather than merely removing a reward term. For a fixed anchor trajectory, the corresponding KL-regularized subproblem can be written as
\begin{equation}
\mathcal L_{\bar m}(\pi\mid\tau^m_{\mathrm{anc}},c)
=\mathbb E_{\tau^{\bar m}\sim\pi^{\bar m}(\cdot\mid\tau^m_{\mathrm{anc}},c)}[R_{\bar m}+R_{\mathrm{sync}}]
-\beta D_{\mathrm{KL}}\!\left(\pi^{\bar m}(\cdot\mid\tau^m_{\mathrm{anc}},c)\|\pi^{\bar m}_{\mathrm{ref}}(\cdot\mid\tau^m_{\mathrm{anc}},c)\right).
\label{eq:conditional-subproblem}
\end{equation}
The video and audio phases solve the two symmetric conditional subproblems. Holding the anchor fixed makes their reward comparisons controlled, while refreshing it between groups prevents training on a single counterpart. The formal relation between these conditional objectives and a joint KL-regularized objective is developed in Appendix~\ref{app:theory}.

For comparison, the joint formulation optimizes a distribution over paired trajectories, $\mathcal L_{\mathrm{joint}}(P)=\mathbb E_P[r(\tau^{\mathrm V},\tau^{\mathrm A},c)]-\beta D_{\mathrm{KL}}(P\|P_0)$. Its samples change both audio and video at once, whereas each AV-GRPO phase conditions on one realized trajectory and changes only the other. Fixing the \emph{complete} trajectory matters because the two towers interact at every denoising step: a fixed final audio or video sample alone would leave the intermediate cross-modal states uncontrolled. The conditional view explains why the within-group reward has a clearer interpretation without assuming that the two generation streams are independent.

\textbf{Trajectory Locking and Tower Freezing.}
The same anchor states are reused throughout each rollout and during its policy update. We freeze $\theta_m$ and update only $\theta_{\bar m}$, retaining gradient and optimizer states for the active tower while the anchor supplies a fixed cross-modal condition. Resampling the anchor for each group exposes the target to varied counterpart trajectories; switching towers every $n$ steps lets both branches improve. For the fixed anchor, the conditional GRPO objective is
\begin{align}
\mathcal{J}_{\bar m}(\theta)={}&\mathbb{E}_{c,\tau^m_{\mathrm{anc}},\{\tau^{\bar m,i}\}_{i=1}^G\sim\pi^{\bar m}_{\theta_{\mathrm{old}}}(\cdot\mid\tau^m_{\mathrm{anc}},c)}
\Bigg[\frac1{GT}\sum_{i=1}^G\sum_t\Big(\nonumber\\
&\min\big(r_t^{\bar m,i}\hat A^{\bar m,i},\,\mathrm{clip}(r_t^{\bar m,i},1-\varepsilon,1+\varepsilon)\hat A^{\bar m,i}\big)
-\beta D_{\mathrm{KL}}\!\left(\pi_\theta^{\bar m}\|\pi_{\mathrm{ref}}^{\bar m}\right)\Big)\Bigg],
\label{eq:avgrpo}
\end{align}
where $r_t^{\bar m,i}$ is the current-to-old conditional transition-probability ratio; both policies and the reference are conditioned on the anchor. Appendix~\ref{app:theory} analyzes an idealized counterpart: under a common reference joint law and contraction assumptions, alternating exact conditional Gibbs kernels converge to the joint KL-regularized optimum.

Freezing the anchor also separates the costs of generating a condition from those of learning under it. An anchor trajectory must still be sampled for each group, but its tower does not require backward activations or optimizer updates in that phase. The target tower sees the same anchor at all $G$ comparisons, allowing reward variation to reflect its sampled trajectories rather than changes in the counterpart. Alternating phases then updates each tower under the other's latest distribution instead of permanently fixing either one.

\subsection{Adaptive Noise Clipping for Robust Sampling}
The diffusion coefficient $\sigma_t=a\sqrt{t/(1-t)}$ in Eq.~(\ref{eq:sde-step}) grows near $t=1$. On our coarse grid, the resulting stochastic increment can overwhelm the latent signal during noising and destabilize training. \textbf{Adaptive Noise Clipping (ANC)} bounds its standard deviation using
\begin{equation}
s_t=\sigma_t\sqrt{|\Delta t|},\qquad
\lambda_t=\min\!\left(1,\frac{\tau}{s_t+\delta}\right),\qquad
\tilde\sigma_t=\lambda_t\sigma_t,
\label{eq:anc}
\end{equation}
where $\tau>0$ is a clipping threshold and $\delta>0$ prevents division by zero. We substitute $\tilde\sigma_t$ for $\sigma_t$ in both the stochastic term and the drift correction of Eq.~(\ref{eq:sde-step}):
\begin{equation}
x_{t+\Delta t}=x_t+\left[v_\theta+\frac{\tilde\sigma_t^2}{2t}\big(x_t+(1-t)v_\theta\big)\right]\Delta t
+\tilde\sigma_t\sqrt{|\Delta t|}\,\epsilon,\quad\epsilon\sim\mathcal N(0,I).
\label{eq:anc-update}
\end{equation}
When clipping is inactive, Eq.~(\ref{eq:anc-update}) reduces to the original update; otherwise the stochastic term scales by $\lambda_t$ and the drift correction by $\lambda_t^2$. We use 15 sampling steps and find ANC stabilizes training at larger noise strengths.

\subsection{Hyperparameter Decoupling and Training Stability}
A shared configuration can limit exploration in one tower while destabilizing the other. AV-GRPO naturally supports independent optimization configurations through its alternating frozen-tower updates, allowing noise strengths and objectives to be tailored to each modality. For LTX-2.3, we use $a_{\mathrm V}=0.02$ and $a_{\mathrm A}=0.8$ in $\sigma_t=a\sqrt{t/(1-t)}$: too little noise limits exploration, whereas too much can break denoising. 


With the original Flow-GRPO objective, the video tower becomes over-exposed after a few iterations. The late denoising stages, which control brightness and fine detail, receive weak gradients; changing only the scalar KL coefficient does not resolve this imbalance. Following LongCat-Video~\citep{team2025longcat}, we reweight the video policy and KL terms at each step:
\begin{equation}
\mathcal{J}_{\text{video}}(\theta) = 
\mathbb{E}\left[
\frac{1}{G} \sum_{i=1}^G
\Big(
\lambda_{\text{policy}}(t, \Delta t) \cdot \mathcal{L}_{\text{policy}}(\theta)
- \beta \, \lambda_{\text{KL}}(t, \Delta t) \cdot D_{\text{KL}}(\pi_\theta^{\text{video}} \| \pi_{\text{ref}}^{\text{video}})
\Big)
\right],
\label{eq:video-loss}
\end{equation}
The weighting functions are defined as:

\begin{equation}
\lambda_{\text{policy}}(t, \Delta t) = \sqrt{\frac{t}{\Delta t (1-t)}}, \qquad
\lambda_{\text{KL}}(t, \Delta t) = \frac{t}{\Delta t (1-t)}.
\label{eq:video-weights}
\end{equation}
The audio tower retains the original weighting. Appendix~\ref{app:losses} gives the full objective and gradient derivation.

\section{Experiment}
\subsection{5DAV Dataset}

To support GRPO optimization for joint audio-video generation models, we construct a \textbf{5D decoupled training set (5DAV)} that is fully disentangled across five dimensions: semantic hierarchy, sound source type, synchronization difficulty, temporal complexity, and instruction granularity. Through Cartesian product of these dimensions, our dataset achieves comprehensive capability coverage with controllable difficulty levels.
The five dimensions are defined as follows:

\begin{itemize}
    \item \textbf{Semantic Hierarchy (W)}: Entity-level (single object/person/animal and its inherent sound), Scene-level (static scene with background atmosphere), Event-level (dynamic interactions and causal events among multiple entities).
    \item \textbf{Sound Source Type (H)}: Human voice, Environmental and object sounds, Music, Mixed sources (two or more types).
    \item \textbf{Synchronization Difficulty (D)}: Irrelevant (audio and video independent), Category matching (audio category matches the scene), Frame-level synchronization (audio and video fully aligned), Physical causal synchronization (e.g., pitch rising as a train approaches).
    \item \textbf{Temporal Complexity (T)}: Single-event (single scene/event, no scene transition), Multi-event (multiple scenes/causal events, with scene transitions).
    \item \textbf{Instruction Granularity (C)}: Weak constraint (core semantics only), Medium constraint (explicit sound source/scene/event), Strong constraint (fine-grained details including precise timing, spatial position, pitch/volume variations).
\end{itemize}

The five-dimensional decomposition yields \(3 \times 4 \times 4 \times 2 \times 3 = 288\) orthogonal categories, with 20 prompts generated per category, resulting in a total of 5,760 data samples. This design promotes comprehensive coverage of audio-video generation application scenarios. Moreover, the sampling ratio of categories within each dimension can be flexibly adjusted according to the requirements of different training stages, enabling the dataset to adapt to diverse optimization objectives and model iteration paces. This design makes the training process traceable and model weaknesses localizable, providing clear guidance for data selection and reward function design.


\begin{figure}[t]
    \centering
    \includegraphics[width=1.0\linewidth]
    {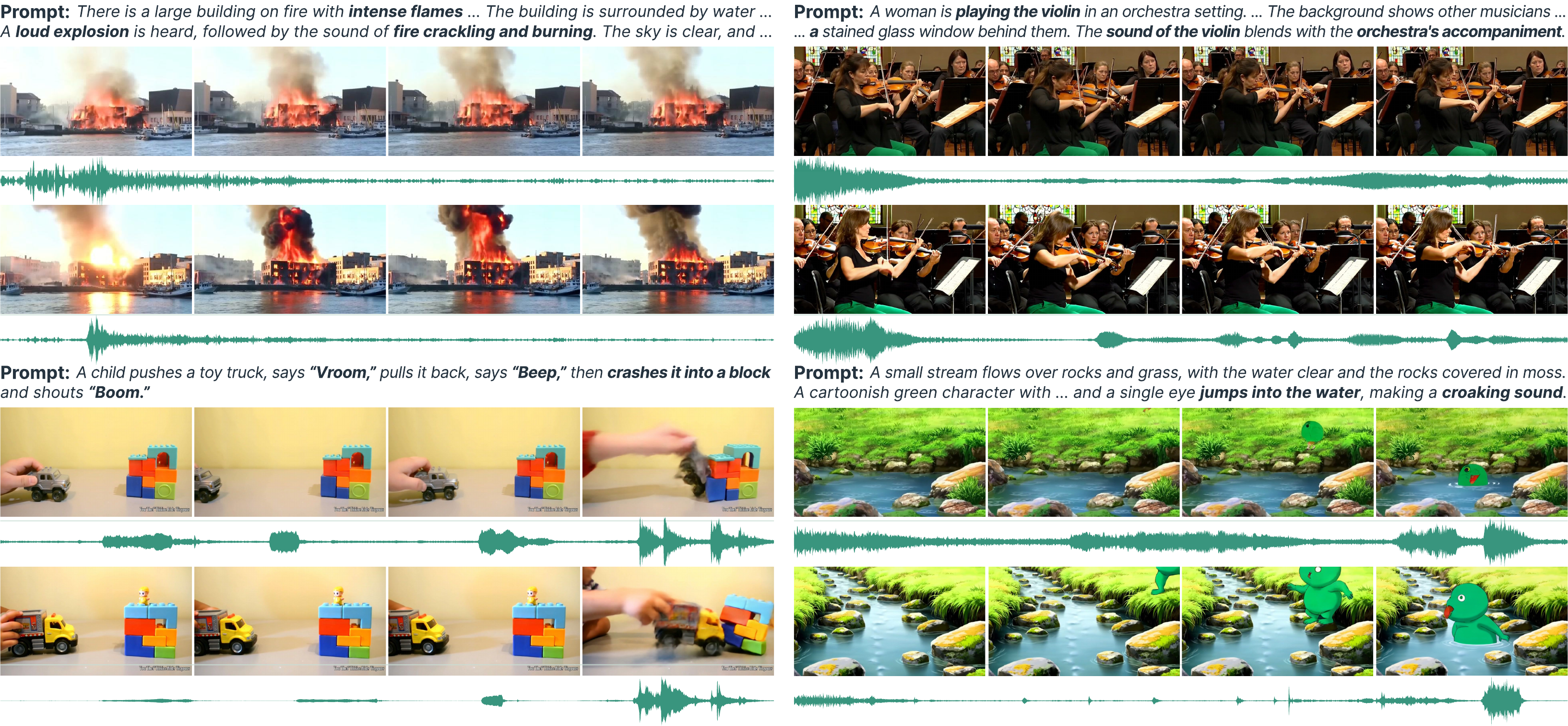}
    \vspace{-3mm}
    \caption{Qualitative comparison of LTX-2.3 and AV-GRPO~(full). For each example, rows 2--3 show the video frames and corresponding audio waveform generated by the original LTX-2.3 model; rows 4--5 show those generated by AV-GRPO (full). Bold text in the prompt highlights the visual events and sound cues that require audio--video synchronization.}
    \label{fig:qualitative_comparison}
    \vspace{-3mm}
\end{figure}

\subsection{Experimental Setup}

\paragraph{Training configuration.}
All LoRA and full-parameter runs use a single node of eight NVIDIA A800 80~GB GPUs. LTX-2.3 22B generates $544\times960$ clips of 97 frames at 24~FPS with 15 denoising steps. The video tower is updated first, and the active tower switches every four steps. Appendix~\ref{app:training} provides the remaining implementation details.

\subsubsection{Reward Models and Reward Composition}
\label{sec:rewards}
For video-tower optimization, we use VideoAlign~\citep{liu2026improving}, CLIP~\citep{radford2021learning}, and DeSync~\citep{iashin2024synchformer} to assess video quality, video--prompt alignment, and audio--video synchronization, respectively. For audio-tower optimization, we use Audiobox Aesthetics~\citep{tjandra2025meta}, CLAP~\citep{wu2023large}, and DeSync to assess audio quality, audio--prompt alignment, and audio--video synchronization, respectively. 

Following GDPO~\citep{Liu2026GDPOGR},  the video score sums standardized VideoAlign, CLIP, and negative DeSync, whereas the audio score sums standardized audio aesthetics, CLAP, and negative DeSync (Since lower DeSync is better, its sign is reversed). We standardize the resulting score again within the group to obtain the trajectory advantage.  For audio, a CLAP guardrail uses only the alignment advantage when the group's average CLAP falls below a threshold; this prevents quality and synchronization gains from masking a loss of prompt fidelity. Appendix~\ref{app:rewards} gives the equations.

\subsubsection{Evaluation Benchmarks}

We adopt JavisBench~\citep{liu2025javisdit} and VABench~\citep{hua2026vabench} as our evaluation benchmarks to comprehensively assess the generated audio-video content.

JavisBench evaluates generated samples across four complementary dimensions: (1) \textbf{Unimodal Generation Fidelity}, measured by Visual Quality (VQ) and Audio Quality (AQ), assessing the perceptual quality of each modality independently; (2) \textbf{Text-Modal Alignment}, where ImageBind similarity, CLIP score, and CLAP score are employed to measure the semantic consistency between the textual prompt and each generated modality; (3) \textbf{Audio-Video Semantic Coherence}, evaluated via Audio-Video ImageBind similarity (AV-IB) and AVH Score to quantify cross-modal semantic alignment; (4) \textbf{Audio-Video Synchronization}, assessed by JavisScore and DeSync to measure temporal correspondence between audio and video streams.

VABench organizes its evaluation into two paradigms. The first paradigm relies on expert models to provide objective perceptual quality assessments, including speech quality and naturalness (SpeechQual\&Nat), audio aesthetics (AudioAesthetic), and lip synchronization accuracy (Lip-Sync). The second paradigm leverages Multimodal Large Language Models (MLLMs) to simulate human judgments on complex audio-video semantics, covering high-level criteria such as global Alignment, Artistic quality (Artistry), and Expressiveness.

\subsection{Main Results and Analysis}

\subsubsection{Quantitative Analysis}

Table \ref{table: JavisBench} and \ref{table: VABench} report the results of AV-GRPO under both LoRA and full fine-tuning on JavisBench and VABench. Compared with LTX-2.3 (22B) and GDPO, AV-GRPO consistently improves nearly all metrics under both settings, with full fine-tuned model attaining the best overall performance.

In terms of perceptual quality, the full fine-tuned model improves AQ on JavisBench from 5.097 to 5.798 and Audio Aesthetics on VABench from 3.319 to 3.631 over LTX-2.3. For semantic alignment, it boosts CLIP on JavisBench from 0.318 to 0.327 and CLAP from 0.408 to 0.468. For cross-modality consistency and synchronization, substantial gains are observed across the board: AVH Score, DeSync, and AV-Align on JavisBench, as well as Lip Sync and DeSync on VABench, all exhibit significant improvements. In all these aspects, AV-GRPO consistently outperforms GDPO by a considerable margin.

The two training regimes show different strengths. Full fine-tuning yields the highest JavisBench AQ and AV-IB scores (5.798 and 0.247), while LoRA attains a slightly higher CLIP score (0.330 versus 0.327) and a lower DeSync value (0.554 versus 0.607). The improvements are therefore broad rather than uniform: LoRA's VQ is 5.816, below the base model's 5.855, and the full model's VABench visual-realism score is 4.395 versus 4.399 for the base model. These exceptions do not change the gains in most quality and alignment measures, but they show that different modalities and training regimes retain distinct tradeoffs.

Compared with GDPO, AV-GRPO~(full) raises JavisBench AV-IB from 0.224 to 0.247 and JavisScore from 0.202 to 0.222, reducing DeSync from 0.708 to 0.607. On VABench, its Lip Sync increases from 1.439 to 1.646 and DeSync decreases from 0.726 to 0.542. 
These comparisons focus on cross-modal measures, where controlling the ref-modality is most relevant to the learning signal.

\begin{table}[]
\centering
\caption{Main results on JavisBench. Best results are in \textbf{bold}, second-best are \underline{underlined}. VQ: Visual Quality, AQ: Audio Quality. ($\uparrow$: higher is better; $\downarrow$: lower is better).}
\scalebox{0.605}[0.605]{
\begin{tabular}{@{}lcccccccccccc@{}}
\toprule
                        &                        & \multicolumn{2}{c}{AV-Quality} & \multicolumn{4}{c}{Text-Consistency} & \multicolumn{2}{c}{AV-Consistency} & \multicolumn{3}{c}{AV-Synchrony} \\
\arrayrulecolor[HTML]{888888}\cmidrule(lr){3-4}\cmidrule(lr){5-8}\cmidrule(lr){9-10}\cmidrule(r{0pt}){11-13}\arrayrulecolor{black}
\multirow{-2}{*}{Model} & \multirow{-2}{*}{size} & VQ$\uparrow$            & AQ$\uparrow$           & TV-IB$\uparrow$  & TA-IB$\uparrow$   & CLIP$\uparrow$    & CLAP$\uparrow$   & AV-IB$\uparrow$          & AVHScore$\uparrow$          & JavisScore$\uparrow$  & DeSync$\downarrow$  & AV-align$\uparrow$ \\ \midrule
\rowcolor[HTML]{EFEFEF} 
\textit{-T2A+A2V}       &                        &                &               &         &         &         &        &                &                   &             &         &          \\
TempoTKn                & 1.3B                   & -              & -             & 0.084   & -       & 0.205   & -      & 0.139          & 0.122             & 0.103       & 1.532   & -        \\
TPoS                    & 1.0B                   & -              & -             & 0.201   & -       & 0.229   & -      & 0.124          & 0.129             & 0.095       & 1.493   & -        \\ \midrule
\rowcolor[HTML]{EFEFEF} 
\textit{-T2V+V2A}       &                        &                &               &         &         &         &        &                &                   &             &         &          \\
ReWaS                   & 0.6B                   & -              & -             & -       & 0.123   & -       & 0.280  & 0.110          & 0.104             & 0.079       & 1.071   & -        \\
See\&Hear               & 0.4B                   & -              & -             & -       & 0.129   & -       & 0.263  & 0.160          & 0.143             & 0.112       & 1.099   & -        \\
FolleyCrafter           & 1.2B                   & -              & -             & -       & 0.149   & -       & 0.383  & 0.193          & 0.186             & 0.151       & 0.952   & -        \\
MMAudio                 & 0.1B                   & -              & -             & -       & 0.160   & -       & 0.407  & 0.198          & 0.182             & 0.150       & 0.849   & -        \\ \midrule
\rowcolor[HTML]{EFEFEF} 
\textit{-T2AV}          &                        &                &               &         &         &         &        &                &                   &             &         &          \\
JavisDiT                & 3.1B                   & 1.291          & 4.478         & 0.263   & 0.143   & 0.302   & 0.391  & 0.197          & 0.179             & 0.154       & 1.039   & -        \\
UniVerse-1              & 6.4B                   & 1.357          & 4.839         & 0.272   & 0.111   & 0.309   & 0.245  & 0.104          & 0.098             & 0.077       & 0.929   & -        \\
JavisDiT++              & 2.1B                   & 1.462          & 5.049         & 0.282   & 0.164   & 0.316   & 0.424  & 0.198          & 0.184             & 0.159       & 0.832   & -        \\
LTX-2.3                 & 22B                    & 5.855          & 5.097         & 0.284   & 0.143   & 0.318   & 0.408  & 0.212          & 0.201             & 0.183       & 0.757   & 0.354    \\ \midrule
LTX-2.3+GDPO            & 22B                    & \underline{5.870}          & 5.406         & 0.283   & 0.165   & 0.313   & 0.431  & 0.224          & 0.217             & 0.202       & 0.708   & 0.360    \\
LTX-2.3+AV-GRPO(LoRA)   & 22B                    & 5.816          & \underline{5.633}         & \textbf{0.289}   & \underline{0.174}   & \textbf{0.330}   & \underline{0.449}  & \underline{0.236}          & \underline{0.235}             & \underline{0.213}       & \textbf{0.554}   & \underline{0.382}    \\
LTX-2.3+AV-GRPO(full)   & 22B                    & \textbf{5.902}          & \textbf{5.798}         & \underline{0.288}   & \textbf{0.185}   & \underline{0.327}   & \textbf{0.468}  & \textbf{0.247}          & \textbf{0.241}            & \textbf{0.222}       & \underline{0.607}   & \textbf{0.398}    \\ \bottomrule
\end{tabular}}
\label{table: JavisBench}
\vspace{-5mm}
\end{table}

\begin{table}[]
\centering
\caption{Main results on VABench. Best results are in \textbf{bold}, second-best are \underline{underlined}. For all metrics except DeSync, higher is better; for DeSync, lower is better.}
\scalebox{0.65}[0.65]{
\begin{tabular}{@{}lcccccccccccc@{}}
\toprule
Model                 & \begin{tabular}[c]{@{}c@{}}Speech\\ Q\&N\end{tabular} & \begin{tabular}[c]{@{}c@{}}Audio\\ Aes\end{tabular} & \begin{tabular}[c]{@{}c@{}}T-V\\ Align\end{tabular} & \begin{tabular}[c]{@{}c@{}}T-A\\ Align\end{tabular} & \begin{tabular}[c]{@{}c@{}}A-V\\ Align\end{tabular} & \begin{tabular}[c]{@{}c@{}}Lip\\ Sync\end{tabular} & DeSync$\downarrow$ & Alignment & \begin{tabular}[c]{@{}c@{}}Express-\\ siveness\end{tabular} & \begin{tabular}[c]{@{}c@{}}Visual\\ Realism\end{tabular} & \begin{tabular}[c]{@{}c@{}}Audio\\ Realism\end{tabular} & Artistry \\ \midrule
LTX-2.3               & 1.383                                                 & 3.319                                               & 0.211                                               & 0.398                                               & 0.227                                               & 1.351                                              & 0.800  & 4.455     & 4.371                                                       & 4.399                                                    & 3.830                                                   & 3.766    \\ \midrule
LTX-2.3+GDPO          & 1.465                                                 & 3.452                                               & 0.210                                               & 0.424                                               & 0.243                                               & 1.439                                              & 0.726  & 4.481     & 4.409                                                       & \textbf{4.413}                                                    & 3.849                                                   & 3.784    \\
LTX-2.3+AV-GRPO(lora) & \underline{1.487}                                                 & \underline{3.553}                                               & \textbf{0.217}                                              & \underline{0.446}                                               & \underline{0.261}                                               & \underline{1.585}                                              & \underline{0.594}  & \textbf{4.512}     & \textbf{4.450}                                                       & \underline{4.402}                                                    & \underline{3.865}                                                   & \underline{3.792}    \\
LTX-2.3+AV-GRPO(full) & \textbf{1.510}                                                 & \textbf{3.631}                                               & \underline{0.215}                                               & \textbf{0.468}                                               & \textbf{0.272}                                               & \textbf{1.646}                                              & \textbf{0.542}  & \underline{4.510}     & \underline{4.434}                                                       & 4.395                                                    & \textbf{3.878}                                                   & \textbf{3.798}    \\ \bottomrule
\end{tabular}}
\label{table: VABench}
\vspace{-5mm}
\end{table}

\subsubsection{Qualitative Analysis}

Figure \ref{fig:qualitative_comparison} 
 showcases several representative cases. As shown in the figure, compared with the original model, AV-GRPO-trained LTX-2.3 achieves improvements across different aspects: per-modality fidelity, modality-text alignment, and cross-modal synchronization.

The examples span event-driven effects, music performance, object sounds, and natural ambience. In each row, the video frames and corresponding waveform are shown together so that visual changes can be compared with the timing and character of the generated audio. These cases complement the aggregate benchmark scores by illustrating the kinds of cross-modal relationships optimized through the shared-anchor comparisons.

\subsection{Ablation Study}
\subsubsection{Dataset}
To demonstrate the effectiveness of our proposed dataset, we conduct a comparative experiment against the well-known audio-video dataset VGGSound~\citep{chen2020vggsound} under full fine-tuning settings. All training configurations are kept identical across experiments. The results on JavisBench-mini are reported in Figure \ref{fig:Figure ablation}.

\begin{figure}[t]
    \centering
    \includegraphics[width=1.0\linewidth]{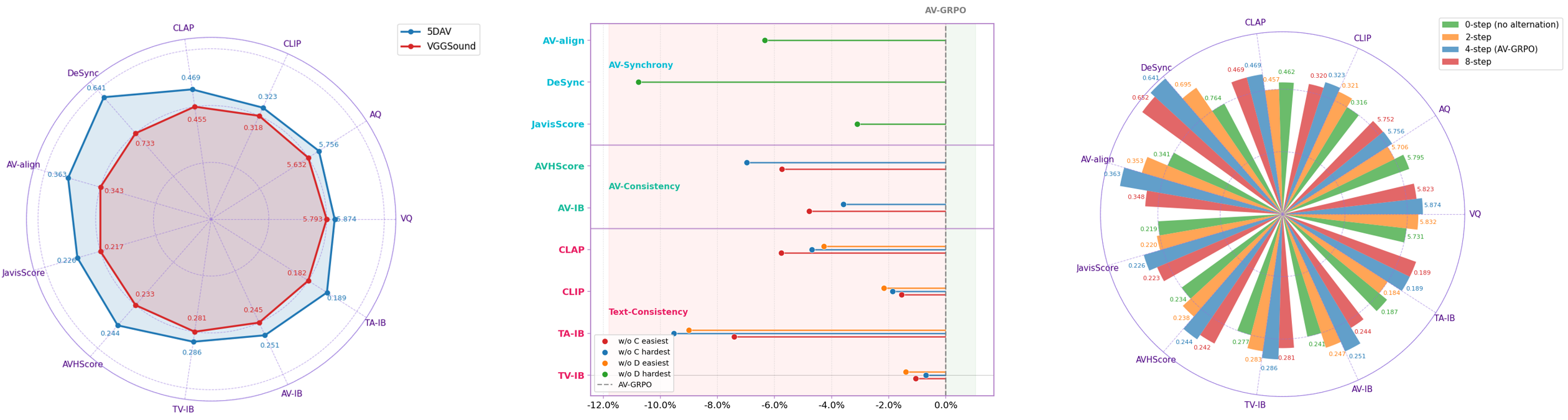}
    \vspace{-4mm}
    \caption{\textbf{Left}: Radar chart comparing performance on the 5DAV and VGGSound datasets. \textbf{Middle}: Metrics are reported as relative percentage changes with respect to the baseline.
    \textbf{Right}: Comparison of model performance with different alternating step intervals.}
    \label{fig:Figure ablation}
    \vspace{-5mm}
\end{figure}

We conduct ablation studies on the difficulty levels of Synchronization Difficulty (\textbf{D}) and Instruction Granularity (\textbf{C}). Specifically, we compare the easiest and most difficult categories within each of these
two dimensions (denoted as \textbf{D/C easiest/hardest}) while keeping all other settings identical. The experimental results are reported in Figure \ref{fig:Figure ablation}. The results demonstrate that the decoupling and categorical partitioning of dimensions and difficulty levels contribute substantially to the model's performance.

The left panel compares the two training sets across quality, text alignment, cross-modal consistency, and synchronization metrics; the middle panel reports the relative change when an easy or hard level of \textbf{C} or \textbf{D} is removed. Together, these views test both the value of the complete category grid and whether its difficulty levels supply distinct training signals. In particular, changes in the synchronization measures under the \textbf{D} removals show why a single aggregate benchmark score is insufficient to assess the data design.

\subsubsection{Alternating Modality-wise Optimization}

To verify the effectiveness of alternating optimization, we conduct ablation studies on whether to apply alternating optimization, and provide experimental results with varying switch intervals (i.e., the number of training steps per modality before switching to the other). All experiments are conducted on JavisBench-mini, as shown in Figure~\ref{fig:Figure ablation}. The results demonstrate that alternating optimization substantially benefits the performance of both towers. Under our current configuration, switching every four steps achieves relatively optimal performance.

The right panel includes no alternation and intervals of two, four, and eight steps. The four-step schedule gives the strongest overall balance across the plotted measures, although individual metrics need not peak at the same interval. This comparison motivates the switching interval used for the main LoRA and full-parameter runs.

\section{Related Work}

\subsection{Joint Audio--Video Generation}

Early audio--visual generation methods typically adopt asymmetric pipelines that synthesize one modality conditioned on the other. For example, FoleyCrafter~\citep{zhang2026foleycrafter} and MMAudio~\citep{cheng2025taming} generate temporally aligned audio from video. While such methods benefit from strong unimodal priors, their asymmetric formulations limit direct bidirectional co-generation. Recent works instead generate audio and video within a unified process. JavisDiT~\citep{liu2025javisdit}, UniVerse-1~\citep{wang2025universe}, Ovi~\citep{low2025ovi} and LTX-2~\citep{hacohen2026ltx} employ interacting dual-stream architectures with cross-modal fusion. These advances primarily focus on architecture, data, and pretraining. Beyond joint generation, video research covers scientific understanding, personalized dialogue, and narrative synthesis~\citep{xu2026scieducator, shi2025pvchat, shi2026one}, while geometric work explores 3D completion, robust depth, and dynamic point clouds~\citep{yan2026comp, yan2025synthetic, zhang20264dpc}. AV-GRPO instead focuses on reward-guided post-training of synchronized audio and video, where joint sampling makes credit assignment difficult.

\subsection{Reinforcement Learning-based Post-training}

Reward-guided post-training has been widely explored for diffusion and flow-based generation. DDPO~\citep{black2023ddpo} and DPOK~\citep{Fan2023DPOKRL} formulate reverse denoising as a sequential decision process and optimize diffusion models using policy gradients. More recently, Flow-GRPO~\citep{liu2026flow} converts deterministic flow sampling into an equivalent stochastic process for online GRPO. DiffusionNFT~\citep{zheng2026diffusionnft} instead performs online reward optimization through the forward process. For multiple objectives, GDPO~\citep{Liu2026GDPOGR} independently normalizes individual rewards before aggregation, but does not resolve cross-modal credit assignment when both modalities vary simultaneously.

\section{Conclusion}\label{sec:conclusion}
We presented AV-GRPO, a modality-anchored reinforcement learning framework for joint audio--video post-training. By combining anchored rollouts with alternating frozen-tower optimization, AV-GRPO disentangles rewards and controls comparison conditions, reduces training overhead, redirects credit assignment, and adapts to modality-specific dynamics and sample difficulty. We also introduced 5DAV, a five-dimensionally decoupled dataset for controllable post-training. Experiments on JavisBench and VABench show consistent improvements over LTX-2.3 and GDPO in perceptual quality, text alignment, and audio--video coherence and synchronization under both LoRA and full fine-tuning. These results demonstrate an effective and scalable approach to improving coupled audio--video generators.
\label{end:main}

\subsubsection*{Acknowledgments}
This work is supported by Shanghai Artificial Intelligence Laboratory.

\nocite{*}

\bibliographystyle{iclr2027_conference}
\bibliography{iclr2027_conference}

\appendix

\newcommand{\X}{\mathcal{X}}
\newcommand{\Y}{\mathcal{Y}}
\newcommand{\Borel}{\mathcal{B}}
\newcommand{\Pspace}{\mathcal{P}}
\newcommand{\KLdiv}{\mathrm{D}_{\mathrm{KL}}}
\newcommand{\TV}{\mathrm{TV}}
\newcommand{\EE}{\mathbb{E}}
\newcommand{\PP}{\mathbb{P}}
\newcommand{\dd}{\,\mathrm{d}}
\newcommand{\Kv}{K_v^*}
\newcommand{\Ka}{K_a^*}
\newcommand{\Kvzero}{K_v^0}
\newcommand{\Kazero}{K_a^0}
\newcommand{\Pzero}{P_0}
\newcommand{\Pstar}{P^*}
\newcommand{\Mv}{M_v}
\newcommand{\Ma}{M_a}
\newcommand{\Vop}{V}
\newcommand{\Wop}{W}
\newcommand{\Fop}{F}
\newcommand{\deltav}{\delta_v}
\newcommand{\deltaa}{\delta_a}
\newcommand{\rhoD}{\rho}
\definecolor{revisionred}{RGB}{200,0,0}
\newcommand{\rev}[1]{{\color{revisionred}#1}}

\newtheorem{lemma}{Lemma}
\newtheorem{theorem}{Theorem}
\newtheorem{proposition}{Proposition}
\newtheorem{corollary}{Corollary}
\theoremstyle{definition}
\newtheorem{assumption}{Assumption}
\newtheorem{remark}{Remark}

\section{Proof of the Convergence and Joint Optimality of AV-GRPO}\label{app:theory}

This section studies KL-regularized reinforcement learning fine-tuning of audio-video two-tower diffusion models under the protocol of ``alternately fixing a single-modality complete denoising trajectory''.
We introduce a common reference joint law \(P_0\), and specify the two single-tower operation kernels as the regular conditional kernels of \(P_0\);
the reward \(r\) is bounded measurable, and \(\beta>0\). We define the exponentially tilted joint law
\(\dd P^*/\dd P_0 \propto e^{r/\beta}\), and construct the two conditional Gibbs optimal kernels \(K_v^*,K_a^*\).
Under the Dobrushin contraction condition \(\delta_v\delta_a<1\), we prove:
\(P^*\) is the unique global optimum of the joint KL-regularized objective;
\(K_v^*,K_a^*\) are the unique optimal kernels of the conditional subproblems, respectively, and are automatically compatible;
the systematic-scan Gibbs chain with kernels \(K_v^*,K_a^*\) has a unique stationary distribution \(P^*\);
starting from any initial joint law, the joint laws of both sampling phases converge geometrically in total variation to \(P^*\), with an explicit convergence rate.
If the model implements \(K_v^*,K_a^*\) via a non-interfering exact kernel oracle, then both the training rollouts and the inference distribution using the same alternating full-trajectory protocol possess the above convergence properties.

\subsection{Problem Setup and Basic Assumptions}

\subsubsection{Spaces and the Reference Joint Law}

Fix the text condition \(c\), which is omitted in the following. Let \(\X\) be the space of complete video denoising trajectories, and \(\Y\) the space of complete audio denoising trajectories.
Assume that \(\X,\Y\) are spaces equipped with fixed Polish topologies, with their Borel \(\sigma\)-algebras denoted by \(\Borel_\X,\Borel_\Y\), respectively.
Thus \(\X\times\Y\) is also a Polish space, and regular conditional probability kernels exist. Let \(\Pspace(\X),\Pspace(\Y),\Pspace(\X\times\Y)\)
denote the corresponding sets of probability measures.

\begin{assumption}[Common reference joint law H0]
There exists a reference joint probability measure \(P_0\in\Pspace(\X\times\Y)\). The actual reference kernel for
``running the video tower with the audio trajectory fixed'', \(K_v^0:\Y\times\Borel_\X\to[0,1]\), and the reference kernel for
``running the audio tower with the video trajectory fixed'', \(K_a^0:\X\times\Borel_\Y\to[0,1]\), are two versions of the regular conditional distributions of \(P_0\):
\[
K_v^0(\cdot\mid y)=P_0(\dd x\mid y),\quad P_{0,Y}\text{-a.s.},\qquad
K_a^0(\cdot\mid x)=P_0(\dd y\mid x),\quad P_{0,X}\text{-a.s.},
\]
and we select a jointly measurable version of them that is defined on all of \(\X,\Y\). Hence
\[
P_0(A\times B)=\int_B K_v^0(A\mid y)\,P_{0,Y}(\dd y)
=\int_A K_a^0(B\mid x)\,P_{0,X}(\dd x).
\]
\end{assumption}

\subsubsection{Reward and the Joint Objective}

Let \(r:\X\times\Y\to\mathbb{R}\) be a bounded Borel measurable function, \(\beta>0\), and let \(M=\|r\|_\infty\). Define the joint objective
\[
J(P)=\EE_P[r]-\beta\,\KL(P\parallel P_0),\qquad P\in\Pspace(\X\times\Y),
\]
with the convention that \(\KL(P\parallel P_0)=\infty\) when \(P\not\ll P_0\), so that \(J(P)=-\infty\).

\subsubsection{Conditional Subproblem Objectives}

Fix \(y\in\Y\); for any \(q\in\Pspace(\X)\), define the video subproblem
\[
L_v(q;y)=\int_\X r(x,y)\,q(\dd x)-\beta\,\KL\big(q\parallel K_v^0(\cdot\mid y)\big).
\]

Fix \(x\in\X\); for any \(q\in\Pspace(\Y)\), define the audio subproblem
\[
L_a(q;x)=\int_\Y r(x,y)\,q(\dd y)-\beta\,\KL\big(q\parallel K_a^0(\cdot\mid x)\big).
\]

\subsection{Exponential Tilting and Conditional Gibbs Optimal Kernels}

\subsubsection{Joint Exponential Tilting}

Define the normalizing constant and the probability measure
\[
Z=\int_{\X\times\Y} e^{r(x,y)/\beta}\,P_0(\dd x,\dd y),\qquad
\frac{\dd P^*}{\dd P_0}(x,y)=Z^{-1}e^{r(x,y)/\beta}.
\]
Since \(r\) is bounded, \(e^{-M/\beta}\le Z\le e^{M/\beta}\), hence \(0<Z<\infty\), and \(P^*\) and \(P_0\) are mutually absolutely continuous.
Their marginals are denoted by \(P^*_X\) and \(P^*_Y\), respectively.

\subsubsection{Conditional Normalizing Constants and Optimal Kernels}

Define

\[
Z_v(y)=\int_\X e^{r(x,y)/\beta}\,K_v^0(\dd x\mid y),\qquad
Z_a(x)=\int_\Y e^{r(x,y)/\beta}\,K_a^0(\dd y\mid x).
\]

By the measurability theorem for kernel integrals, \(Z_v,Z_a\) are measurable, and
\(e^{-M/\beta}\le Z_v(y),Z_a(x)\le e^{M/\beta}\), so the denominators are everywhere positive and finite. Define the optimal conditional kernels
\[
K_v^*(\dd x\mid y)=Z_v(y)^{-1}e^{r(x,y)/\beta}K_v^0(\dd x\mid y),
\]
\[
K_a^*(\dd y\mid x)=Z_a(x)^{-1}e^{r(x,y)/\beta}K_a^0(\dd y\mid x).
\]
The integrands are jointly measurable, and the normalizing constants are measurable and everywhere positive; therefore \(K_v^*,K_a^*\) are probability kernels defined on all conditional states.

\subsection{Joint KL Variational Identity}

\begin{lemma}[Joint Gibbs variational identity]\label{lem:joint-gibbs}
For any \(P\in\Pspace(\X\times\Y)\), in the extended real-valued sense,
\[
J(P)=\beta\log Z-\beta\,\KL(P\parallel P^*).
\]
Therefore \(P^*\) is the unique global optimum of \(J\) over all joint probability measures, with optimal value \(\beta\log Z\).
\end{lemma}

\begin{proof}
Since \(\dd P^*/\dd P_0=e^{r/\beta}/Z\) has strictly positive finite upper and lower bounds, \(P^*\) and \(P_0\) are equivalent.
Therefore \(P\ll P_0\) if and only if \(P\ll P^*\). If this condition does not hold, both sides are \(-\infty\), and the identity holds.
Assume \(P\ll P_0\). By the Radon--Nikodym chain rule, \(P\)-a.s.\ we have
\[
\log\frac{\dd P}{\dd P^*}=\log\frac{\dd P}{\dd P_0}-\frac{r}{\beta}+\log Z.
\]
Integrating both sides against \(P\) and multiplying by \(\beta\), we obtain
\[
\beta\,\KL(P\parallel P^*)=\beta\,\KL(P\parallel P_0)-\EE_P[r]+\beta\log Z.
\]
Because \(r\) is bounded, the two KL terms are either simultaneously finite or simultaneously \(\infty\), so the above identity is unambiguous in the extended real-valued sense.
Rearranging yields
\[
J(P)=\beta\log Z-\beta\,\KL(P\parallel P^*)\le\beta\log Z.
\]
Finally, \(\KL(P\parallel P^*)=0\) if and only if \(P=P^*\). Hence \(P^*\) is the unique optimum.
\end{proof}

\subsection{Optimal Kernels of the Conditional Subproblems}

\begin{lemma}[Conditional Gibbs optimality]\label{lem:conditional-gibbs}
For each \(y\in\Y\), \(K_v^*(\cdot\mid y)\) is the unique global optimum of \(L_v(\cdot;y)\) over \(\Pspace(\X)\),
with optimal value \(\beta\log Z_v(y)\). For each \(x\in\X\), \(K_a^*(\cdot\mid x)\) is the unique global optimum of the audio subproblem,
with optimal value \(\beta\log Z_a(x)\).
\end{lemma}

\begin{proof}
Fix \(y\), and replace \(P_0,P^*,Z,r\) by \(K_v^0(\cdot\mid y),K_v^*(\cdot\mid y),Z_v(y),r(\cdot,y)\), respectively.
The Radon--Nikodym computation of Lemma \ref{lem:joint-gibbs} applies verbatim, giving
\[
L_v(q;y)=\beta\log Z_v(y)-\beta\,\KL\big(q\parallel K_v^*(\cdot\mid y)\big).
\]
KL is nonnegative and vanishes only when \(q=K_v^*(\cdot\mid y)\), so the conclusion holds. The audio side is entirely symmetric.
\end{proof}

\subsection{Automatic Compatibility of the Two Optimal Conditional Kernels}

\begin{lemma}[Marginals and regular conditional kernels of \(P^*\)]\label{lem:marginal-conditional}
\[
P^*_Y(\dd y)=\frac{Z_v(y)}{Z}\,P_{0,Y}(\dd y),\qquad
P^*_X(\dd x)=\frac{Z_a(x)}{Z}\,P_{0,X}(\dd x).
\]
Hence \(P^*_Y\) is equivalent to \(P_{0,Y}\), and \(P^*_X\) is equivalent to \(P_{0,X}\).
The selected \(K_v^*,K_a^*\) are, respectively, versions of the video and audio regular conditional distributions of \(P^*\).
\end{lemma}

\begin{proof}
For any \(B\in\Borel_\Y\), using the decomposition of \(P_0\) via \(K_v^0\):
\[
P^*_Y(B)=Z^{-1}\int_B\int_\X e^{r/\beta}K_v^0(\dd x\mid y)P_{0,Y}(\dd y)
=Z^{-1}\int_B Z_v(y)P_{0,Y}(\dd y).
\]
Thus \(P^*_Y(\dd y)=Z^{-1}Z_v(y)P_{0,Y}(\dd y)\). \(Z_v\) has strictly positive finite upper and lower bounds, so the two audio marginals are equivalent.
The video marginal formula follows analogously.

Next, taking \(A\in\Borel_\X,B\in\Borel_\Y\), we have
\[
\int_B K_v^*(A\mid y)P^*_Y(\dd y)
=Z^{-1}\int_B\int_A e^{r/\beta}K_v^0(\dd x\mid y)P_{0,Y}(\dd y)
=P^*(A\times B).
\]
This is precisely the definition of \(K_v^*\) being a regular conditional distribution of \(P^*\). \(K_a^*\) is symmetric.
\end{proof}

\subsection{Alternating Kernels, Marginal Operators, and the Systematic-Scan Chain}

\subsubsection{Kernel--Measure Mixing}

For \(\mu\in\Pspace(\Y)\), \(\nu\in\Pspace(\X)\), define the joint distributions of the two phases
\[
\Mv\mu(\dd x,\dd y)=K_v^*(\dd x\mid y)\mu(\dd y),\qquad
\Ma\nu(\dd x,\dd y)=\nu(\dd x)K_a^*(\dd y\mid x).
\]
Define the corresponding marginal operators
\[
\Vop\mu(\dd x)=\int_\Y K_v^*(\dd x\mid y)\mu(\dd y),\qquad
\Wop\nu(\dd y)=\int_\X K_a^*(\dd y\mid x)\nu(\dd x).
\]
Here the \(Y\)-marginal of \(\Mv\mu\) is \(\mu\), and its \(X\)-marginal is \(\Vop\mu\); the \(X\)-marginal of \(\Ma\nu\) is \(\nu\), and its \(Y\)-marginal is \(\Wop\nu\).
Let \(\Fop=\Wop\Vop:\Pspace(\Y)\to\Pspace(\Y)\).

\subsubsection{One Complete Gibbs Sweep}

Starting from any joint law \(Q\in\Pspace(\X\times\Y)\), first keep \(y\) and resample \(x'\sim K_v^*(\cdot\mid y)\),
then keep \(x'\) and resample \(y'\sim K_a^*(\cdot\mid x')\). Denote one round of transition by \(G\); then
\[
QG=\Ma(\Vop Q_Y),\qquad (QG)_Y=\Wop\Vop Q_Y=\Fop Q_Y.
\]
In particular, the joint law after one round depends only on the \(Y\)-marginal of the input joint law.

\subsubsection{Invariance of \(P^*\)}

By Lemma \ref{lem:marginal-conditional}, \(P^*=\Mv P^*_Y=\Ma P^*_X\), and moreover \(P^*_X=\Vop P^*_Y\), \(P^*_Y=\Wop P^*_X\). Therefore
\[
P^*G=\Ma(\Vop P^*_Y)=\Ma P^*_X=P^*.
\]
So \(P^*\) is a stationary distribution of the alternating Gibbs chain. This step requires no ergodicity or convergence assumptions.

\subsection{Dobrushin Contraction and the Basic Contraction Lemma}

\subsubsection{Total Variation and Dobrushin Coefficients}

Define
\[
\|\mu-\mu'\|_{\TV}=\sup_A|\mu(A)-\mu'(A)|\in[0,1],
\]
\[
\deltav=\sup_{y,y'}\|K_v^*(\cdot\mid y)-K_v^*(\cdot\mid y')\|_{\TV},\qquad
\deltaa=\sup_{x,x'}\|K_a^*(\cdot\mid x)-K_a^*(\cdot\mid x')\|_{\TV}.
\]

\begin{assumption}[Verifiable weak coupling/contraction condition H1]
\[
\rhoD=\deltav\deltaa<1.
\]
This condition is used to derive uniqueness and convergence.
The supremum here is taken with respect to the full-domain conditional kernel versions selected above; H1 is an assumption on these versions.
\end{assumption}

\subsubsection{Markov Kernel Contraction Lemma}

\begin{lemma}[Dobrushin contraction]\label{lem:dobrushin}
Let \(K:S\times\Borel_T\to[0,1]\) be a probability kernel, and \(\delta(K)=\sup_{z,z'}\|K(\cdot\mid z)-K(\cdot\mid z')\|_{\TV}\).
Then for any \(\eta,\eta'\in\Pspace(S)\),
\[
\|\eta K-\eta'K\|_{\TV}\le\delta(K)\|\eta-\eta'\|_{\TV}.
\]
\end{lemma}

\begin{proof}
Let \(\sigma=\eta-\eta'\). If \(t=\|\eta-\eta'\|_{\TV}=0\), the conclusion is obvious. Otherwise the Jordan decomposition of \(\sigma\) satisfies
\(\sigma^+=tp,\sigma^-=tq\), where \(p,q\) are probability measures. For any \(C\in\Borel_T\), let \(f_C(z)=K(C\mid z)\); then
\[
|(\eta K-\eta'K)(C)|=t\left|\int f_C\,\dd p-\int f_C\,\dd q\right|
\le t(\sup_z f_C(z)-\inf_z f_C(z))\le t\delta(K).
\]
Taking the supremum over \(C\) yields the conclusion.
\end{proof}

\subsubsection{Application to \(V,W\)}

\[
\|\Vop\mu-\Vop\mu'\|_{\TV}\le\deltav\|\mu-\mu'\|_{\TV},\qquad
\|\Wop\nu-\Wop\nu'\|_{\TV}\le\deltaa\|\nu-\nu'\|_{\TV},
\]
\[
\|\Fop\mu-\Fop\mu'\|_{\TV}\le\rhoD\|\mu-\mu'\|_{\TV},\qquad \rhoD=\deltav\deltaa<1.
\]

\subsection{Existence, Uniqueness, and Identification of the Self-Consistent Marginals}

\begin{lemma}[Unique self-consistent marginals]\label{lem:self-consistent-marginals}
Under H1, \(\Fop=\Wop\Vop\) is a contraction mapping on the complete metric space \((\Pspace(\Y),\|\cdot\|_{\TV})\).
Hence there exists a unique \(\mu^*\in\Pspace(\Y)\) satisfying \(\Fop\mu^*=\mu^*\). Let \(\nu^*=\Vop\mu^*\);
then \((\mu^*,\nu^*)\) is the unique solution of the equations \(\mu=\Wop\nu,\ \nu=\Vop\mu\).
\end{lemma}

\begin{proof}
Finite signed measures form a Banach space under the total variation norm, and \(\Pspace(\Y)\) is a closed subset thereof, so \((\Pspace(\Y),\TV)\) is complete.
By Lemma \ref{lem:dobrushin}, the Lipschitz constant of \(\Fop\) is at most \(\rhoD<1\). The Banach fixed point theorem yields the unique fixed point \(\mu^*\);
and for any \(\mu_0\in\Pspace(\Y)\),
\[
\|\Fop^n\mu_0-\mu^*\|_{\TV}\le\rhoD^n\|\mu_0-\mu^*\|_{\TV}.
\]
Let \(\nu^*=\Vop\mu^*\); then \(\Wop\nu^*=\Wop\Vop\mu^*=\Fop\mu^*=\mu^*\), so \((\mu^*,\nu^*)\) is self-consistent.
Conversely, any self-consistent pair \((\mu,\nu)\) satisfies \(\mu=\Wop\nu=\Wop\Vop\mu=\Fop\mu\), hence \(\mu=\mu^*\);
and then \(\nu=\Vop\mu=\Vop\mu^*=\nu^*\).
\end{proof}

\begin{remark}[Identification with the marginals of the jointly optimal distribution]
By Lemma \ref{lem:marginal-conditional}, the two marginals of \(P^*\) satisfy
\[
P^*_X=\Vop P^*_Y,\qquad P^*_Y=\Wop P^*_X=\Wop\Vop P^*_Y=\Fop P^*_Y.
\]
So \(P^*_Y\) is a fixed point of \(\Fop\). By uniqueness, \(\mu^*=P^*_Y\), \(\nu^*=P^*_X\).
Thus the self-consistent marginals not only exist and are unique, but are exactly the marginals of the jointly KL-optimal distribution.
\end{remark}

\subsection{The Unique Stationary Joint Distribution}

\begin{lemma}[Unique stationary distribution of the alternating chain]\label{lem:unique-stationary}
Under H0, bounded reward, and H1, the stationary distribution of the systematic-scan Gibbs chain \(G\) exists and is unique, and the unique stationary distribution is \(P^*\).
\end{lemma}

\begin{proof}
Existence has already been proved in A.6.3: \(P^*G=P^*\).

We now prove uniqueness. Let \(Q\) be any stationary distribution, i.e., \(Q=QG\). By the structure of one round of updates,
\[
Q=\Ma(\Vop Q_Y),\qquad Q_Y=\Wop\Vop Q_Y=\Fop Q_Y.
\]
So \(Q_Y\) is a fixed point of \(\Fop\). Lemma \ref{lem:self-consistent-marginals} gives
\(Q_Y=\mu^*=P^*_Y\). Substituting back into the joint distribution formula:
\[
Q=\Ma(\Vop\mu^*)=\Ma\nu^*=\Ma P^*_X=P^*.
\]
Therefore \(P^*\) is the unique stationary distribution.
\end{proof}

\begin{remark}[Strict distinction among optimality, stationarity, and convergence]
That \(P^*\) is the unique joint optimum follows from the KL variational identity and does not require \(\rhoD<1\);
the conditional compatibility of \(K_v^*,K_a^*\) follows from the common \(P_0\) and the common reward tilting, and does not require \(\rhoD<1\);
that \(P^*\) is a stationary distribution follows from the two conditional kernels of \(P^*\), and does not require \(\rhoD<1\);
uniqueness of the stationary distribution follows from the strict contraction of \(\Fop=\Wop\Vop\), and requires \(\rhoD<1\);
geometric convergence from any initial value follows from Banach iteration and the TV isometry, and requires \(\rhoD<1\).
\end{remark}

\subsection{Geometric Rollout Convergence from Arbitrary Initial Laws}

\subsubsection{The Two Sampling Phases}

Take any initial joint law \(Q_0\), and let \(\mu_0=Q_{0,Y}\). For \(n\ge1\), recursively define
\[
Q_n^v=\Mv\mu_{n-1},\qquad \nu_n=\Vop\mu_{n-1},
\]
\[
Q_n^a=\Ma\nu_n,\qquad \mu_n=\Wop\nu_n=\Fop\mu_{n-1}.
\]
\(Q_n^v\) is the joint law after the video resampling; \(Q_n^a\) is the joint law after the subsequent audio resampling, and is also the chain distribution after one complete sweep.

\subsubsection{TV Isometric Embedding of the Mixing Operators}

\(\Mv\) keeps \(y\) as is and only randomly generates \(x\). Hence the non-expansiveness of Markov kernels gives
\(\|\Mv\mu-\Mv\mu'\|_{\TV}\le\|\mu-\mu'\|_{\TV}\), while taking the \(Y\)-marginal gives the reverse inequality. Thus
\[
\|\Mv\mu-\Mv\mu'\|_{\TV}=\|\mu-\mu'\|_{\TV}.
\]
Similarly, \(\|\Ma\nu-\Ma\nu'\|_{\TV}=\|\nu-\nu'\|_{\TV}\).

\begin{theorem}[Geometric TV convergence of the two phases]\label{thm:phase-convergence}
Let \(\rhoD=\deltav\deltaa<1\). Then for any \(Q_0\) and \(n\ge1\):
\[
\|\mu_n-P^*_Y\|_{\TV}\le\rhoD^n\|\mu_0-P^*_Y\|_{\TV},
\]
\[
\|\nu_n-P^*_X\|_{\TV}\le\deltav\rhoD^{n-1}\|\mu_0-P^*_Y\|_{\TV},
\]
\[
\|Q_n^v-P^*\|_{\TV}\le\rhoD^{n-1}\|\mu_0-P^*_Y\|_{\TV},
\]
\[
\|Q_n^a-P^*\|_{\TV}\le\deltav\rhoD^{n-1}\|\mu_0-P^*_Y\|_{\TV}.
\]
When \(\rhoD=0\) and \(n=1\), we take \(\rhoD^0=1\) by convention.
\end{theorem}

\begin{proof}
The first inequality is the Banach iteration. The second follows from \(\nu_n=\Vop\mu_{n-1}\) and the \(\deltav\)-contraction of \(\Vop\).
Since \(P^*=\Mv P^*_Y=\Ma P^*_X\), combining the two isometry identities yields the third and fourth inequalities.
\end{proof}

\subsection{Main Theorem}

\begin{theorem}[Idealized AV-GRPO Gibbs theorem]\label{thm:main}
Fix \(c\). Let \(\X,\Y\) be Polish trajectory spaces, and suppose H0 holds; \(r\) is bounded measurable, and \(\beta>0\).
Define \(P^*,K_v^*,K_a^*\) from \(P_0\) and \(r\). Assume the selected full-domain versions satisfy \(\rhoD=\deltav\deltaa<1\). Then:
\begin{enumerate}[leftmargin=*]
\item \(P^*\) is the unique global optimum of the joint objective \(J(P)=\EE_P[r]-\beta\KL(P\parallel P_0)\), with optimal value \(\beta\log Z\);
\item \(K_v^*,K_a^*\) are, respectively, the unique optimal kernels of the conditional KL subproblems when the counterpart trajectory is fixed, and they are two versions of the regular conditional distributions of \(P^*\);
\item the alternating Gibbs chain has a unique stationary distribution \(P^*\);
\item the self-consistent marginal equations \(\mu=\Wop\nu,\ \nu=\Vop\mu\) have a unique solution \((P^*_Y,P^*_X)\);
\item starting from any initial joint law, the joint laws of the video phase and the audio phase both converge to \(P^*\) in total variation at the explicit rates of Theorem \ref{thm:phase-convergence}.
\end{enumerate}
\end{theorem}

\begin{proof}
Conclusion (1) is Lemma \ref{lem:joint-gibbs}. Conclusion (2) follows from Lemma \ref{lem:conditional-gibbs}
and Lemma \ref{lem:marginal-conditional}. The invariance of \(P^*\) follows from A.6.3;
Lemma \ref{lem:unique-stationary} then uses Dobrushin contraction to prove its uniqueness, giving (3).
Lemma \ref{lem:self-consistent-marginals} identifies the unique self-consistent marginals, giving (4).
Finally, Theorem \ref{thm:phase-convergence} gives the geometric TV convergence of both phases, giving (5).
All conclusions have been proved by the preceding lemmas.
\end{proof}

\subsection{Connection with Alternating Training}

\begin{assumption}[exact-kernel non-interference oracle H2]
The video tower policy class contains the selected full-domain kernel \(K_v^*\), and the audio tower policy class contains \(K_a^*\).
The video update oracle returns, for each \(y\), the unique optimal kernel \(K_v^*(\cdot\mid y)\) of Lemma \ref{lem:conditional-gibbs};
the audio update oracle returns, for each \(x\), \(K_a^*(\cdot\mid x)\).
Furthermore, we assume independently that updating the parameters of either tower does not change the conditional kernel already realized by the other tower.
\end{assumption}

\begin{corollary}[Idealized training--inference consistency]
Under the main theorem and H2, after the video tower and the audio tower each complete one exact update, the two policy kernels are fixed as \(K_v^*,K_a^*\), respectively.
Thereafter, the joint distributions of the video phase and the audio phase of the training rollouts, as well as the inference distribution adopting the same full-trajectory alternating resampling protocol,
all converge geometrically in TV to \(P^*\) from any initial joint law. The limiting joint distribution attains the global optimal value \(\beta\log Z\) of \(J\).
\end{corollary}

\begin{proof}
By Lemma \ref{lem:conditional-gibbs}, the unique optimal solutions of the pointwise conditional subproblems are \(K_v^*\) and \(K_a^*\), respectively.
H2 guarantees that the two towers realize these full-domain kernels after updating, and that subsequent updates do not change the other already-realized kernel.
Therefore, after each tower has been updated once, all subsequent sampling is described exactly by the fixed operators \(\Vop,\Wop,\Mv,\Ma\).
Theorem \ref{thm:main}(5) directly gives the geometric TV convergence of both phases and the inference chain to \(P^*\);
Theorem \ref{thm:main}(1) gives the objective value of \(P^*\).
\end{proof}


\subsection{Conclusion}

Under the common reference joint law H0, bounded reward, \(\beta>0\), and the Dobrushin contraction condition \(\rhoD=\deltav\deltaa<1\),
this section has proved:
\begin{enumerate}[leftmargin=*]
\item the exponentially tilted joint law \(P^*\) is the unique global optimum of the joint KL-regularized objective, with optimal value \(\beta\log Z\);
\item the two conditional Gibbs optimal kernels \(K_v^*,K_a^*\) are the unique optimal kernels of the conditional subproblems, respectively, are automatically compatible, and are the regular conditional distributions of \(P^*\);
\item the systematic-scan Gibbs chain with kernels \(K_v^*,K_a^*\) has a unique stationary distribution \(P^*\);
\item the self-consistent marginal equations have a unique solution, which is exactly the two marginals of \(P^*\);
\item starting from any initial joint law, the joint laws of the two sampling phases converge geometrically in total variation to \(P^*\), with explicit rates;
\item if the model implements \(K_v^*,K_a^*\) via a non-interfering exact-kernel oracle, then both the training rollouts and the inference distribution possess the same geometric convergence properties.
\end{enumerate}
The proof is closed under the idealized assumptions: optimality comes from the KL variational identity, compatibility comes from the common reference joint law,
and uniqueness and convergence come from Dobrushin contraction.

\section{Training Details}\label{app:training}
All training runs, including LoRA and full fine-tuning, are conducted on a single 8-GPU A800 (80 GB) node.
All prompts in the training dataset are fully randomly shuffled (rather than sampled in category-ordered sequences).
Given a sampled prompt, the LTX-2.3 model generates audio-video clips with resolution $544\times960$, $97$ frames, and $24$~FPS via 15-step denoising inference using the default random seed ($42$), producing $8$ samples per prompt. The composition of the reward used to evaluate each audio-video sample is described in Appendix~\ref{app:rewards}. For LoRA training, the learning rate starts from $3\times 10^{-6}$ and follows a cosine decay schedule over 1440 steps in total ( i.e., all data is seen once by each of the two towers: 5760 × 2/8 = 1440), with training stopped early at 480 steps. The noise level for sampling of the video tower is set to $0.02$ with a KL coefficient of $0.01$; for the audio tower, the sampling noise level is $0.8$ and the KL coefficient is $0.002$.
GDPO is trained for 480 steps under identical hyperparameters and the same data ordering. Since GDPO directly generates paired audio–video samples and updates both towers simultaneously, the same data is used to train each tower twice. For full fine-tuning, the learning rate starts at $1\times 10^{-6}$ and follows a cosine decay schedule defined over 1440 steps (keeping the same total step count as LoRA), with training stopped early at 480 steps. The noise levels and KL coefficients for both towers are kept identical to those used in LoRA training. Alternating optimization is adopted: training switches between the two towers every $4$ steps, with $8$ prompts consumed per step and the video tower updated first. Within each round, the video tower is first trained on $4\times 8 = 32$ prompts. Training then switches to the audio tower, which is updated using the same set of $32$ prompts, before proceeding to the next round with fresh data.

\section{Flow-GRPO and LongCat-Video Framework}\label{app:losses}
In this section, we present the complete details of the Flow-GRPO~\citep{liu2026flow} and the LongCat-Video~\citep{team2025longcat} framework.
\subsection{Flow-GRPO}
Flow-GRPO converts the deterministic flow ODE into an equivalent SDE that preserves identical marginal distributions with the original model at all timesteps. The reversed SDE for rectified flow is formulated as:
\begin{equation}
d\boldsymbol{x}_t = \left[\boldsymbol{v}_t(\boldsymbol{x}_t)+\frac{\sigma_t^2}{2t}\big(\boldsymbol{x}_t+(1-t)\boldsymbol{v}_t(\boldsymbol{x}_t)\big)\right]dt+\sigma_t d\boldsymbol{w}.
\end{equation}
where $d\boldsymbol{w}$ denotes the Wiener process increment, and $\sigma_t$ is the time-dependent noise schedule. Flow-GRPO sets $\sigma_t=a\sqrt{\frac{t}{1-t}}$, where scalar hyperparameter $a$ controls the stochastic intensity. Applying Euler-Maruyama discretization to the SDE yields the iterative update formula:
\begin{equation}
\boldsymbol{x}_{t+\Delta t}=
\boldsymbol{x}_t+\left[\boldsymbol{v}_\theta(\boldsymbol{x}_t,t)+\frac{\sigma_t^2}{2t}\big(\boldsymbol{x}_t+(1-t)\boldsymbol{v}_\theta(\boldsymbol{x}_t,t)\big)\right]\Delta t+\sigma_t\sqrt{\Delta t}\,\boldsymbol{\epsilon},
\end{equation}
with $\boldsymbol{\epsilon}\sim \mathcal{N}(0,I)$ introducing Gaussian noise. From Equation (2), the transition distribution $\pi_\theta(\boldsymbol{x}_{t+\Delta t}\mid\boldsymbol{x}_t,c)$ is an isotropic Gaussian distribution. Thus, the KL divergence between current policy and reference policy admits a closed-form solution:
\begin{equation}
D_{\mathrm{KL}}(\pi_\theta\Vert\pi_{\text{ref}})=
\frac{\Delta t}{2}\left(\frac{\sigma_t(1-t)}{2t}+\frac{1}{\sigma_t}\right)^2\|\boldsymbol{v}_\theta(\boldsymbol{x}_t,t)-\boldsymbol{v}_{\text{ref}}(\boldsymbol{x}_t,t)\|^2.
\end{equation}

Flow-GRPO adopts the group-relative formulation for advantage estimation. Given condition $c$, the model samples a group of $G$ reverse trajectories. All timesteps of one trajectory share the identical advantage computed via intra-group reward normalization:
\begin{equation}
\hat A_t^i=\frac{R(\boldsymbol{x}_0^i,c)-\mathrm{mean}(\{R(\boldsymbol{x}_0^i,c)\}_{i=1}^G)}{\mathrm{std}(\{R(\boldsymbol{x}_0^i,c)\}_{i=1}^G)}.
\end{equation}

The importance sampling ratio characterizes the probability ratio between old and new policy:
\begin{equation}
r_t^i(\theta)=\frac{p_\theta(\boldsymbol{x}_{t-1}^i\mid \boldsymbol{x}_t^i,c)}{p_{\theta_{\text{old}}}(\boldsymbol{x}_{t-1}^i\mid \boldsymbol{x}_t^i,c)}.
\end{equation}

Flow-GRPO optimizes the policy by maximizing the following objective:
\begin{multline}
\mathcal{J}_{\text{Flow-GRPO}}(\theta)
=\mathbb{E}_{c,\pi^{\text{old}},\{\boldsymbol{x}^i\}_{i=1}^G}
\Bigg[
\frac{1}{G}\sum_{i=1}^{G}\frac{1}{T}\sum_{t=0}^{T-1}
\Big(
\min\big(r_t^i\hat A_t^i,\,\mathrm{clip}(r_t^i,1-\varepsilon,1+\varepsilon)\hat A_t^i\big) \\
-\beta D_{\mathrm{KL}}(\pi_\theta\Vert\pi_{\text{ref}})
\Big)
\Bigg].
\label{eq:flow_grpo_obj}
\end{multline}

\subsection{LongCat-Video}

The gradient of the policy loss $\mathcal{L}_{\text{policy}}(\theta)=r_t^i(\theta)\hat{A}_t^i$ with respect to parameters $\theta$ can be derived. The gradient computation proceeds as follows:
\begin{equation}
\nabla_{\theta}\mathcal{L}_{\text{policy}}(\theta)=\hat{A}_t^i\nabla_{\theta}r_t^i(\theta).
\end{equation}
\begin{equation}
\nabla_{\theta}r_t^i(\theta)=\frac{p_{\theta}\left(x_{t-1}^i \mid x_t^i,c\right)}{p_{\theta_{\text{old}}}\left(x_{t-1}^i \mid x_t^i,c\right)}
\nabla_{\theta}\log p_{\theta}\left(x_{t-1}^i \mid x_t^i,c\right)
=\nabla_{\theta}\log p_{\theta}\left(x_{t-1}^i \mid x_t^i,c\right).
\end{equation}

Combining these results yields the policy gradient:
\begin{equation}
\nabla_{\theta}\mathcal{L}_{\text{policy}}(\theta)=\hat{A}_t^i r_t^i(\theta)\nabla_{\theta}\log p_{\theta}\left(x_{t-1}^i \mid x_t^i,c\right).
\label{eq:26}
\end{equation}

The score function $\nabla_{\theta}\log p_{\theta}\left(x_{t-1}\mid x_{t},c\right)$ is computed next. The conditional distribution is Gaussian:
\begin{equation}
p_{\theta}\left(x_{t-1}\mid x_{t},c\right)=\mathcal{N}\left(x_{t-1};\mu_{\theta}\left(x_{t},t,c\right),\sigma_t^2\Delta t\mathbf{I}\right).
\end{equation}
\begin{equation}
\nabla_{\theta}\log p_{\theta}=\frac{1}{\sigma_t^2\Delta t}\left(x_{t-1}-\mu_{\theta}\right)\cdot\nabla_{\theta}\mu_{\theta}.
\end{equation}

From the SDE sampling process, the following reparameterization holds:
\begin{equation}
x_{t-1}=\mu_{\theta}+\sigma_t\sqrt{\Delta t}\epsilon,\quad \epsilon\sim\mathcal{N}(0,\mathbf{I}),
\end{equation}

Substituting:
\begin{equation}
\nabla_{\theta}\log p_{\theta}=\frac{1}{\sigma_t^2\Delta t}\left(\sigma_t\sqrt{\Delta t}\epsilon\right)\cdot\nabla_{\theta}\mu_{\theta}=\frac{1}{\sigma_t\sqrt{\Delta t}}\epsilon\cdot\nabla_{\theta}\mu_{\theta}.
\end{equation}
\begin{equation}
\mu_{\theta}=x_t+\left[v_{\theta}\left(x_t,t,c\right)+\frac{\sigma_t^2}{2t}\left(x_t+(1-t)v_{\theta}\left(x_t,t,c\right)\right)\right]\left(-\Delta t\right)
\label{eq:27}
\end{equation}

Simplifying the drift term:
\begin{equation}
\begin{aligned}
\text{drift} &= v_{\theta}+\frac{\sigma_t^2}{2t}x_t+\frac{\sigma_t^2}{2t}(1-t)v_{\theta} \\
&= v_{\theta}\left(1+\frac{\sigma_t^2(1-t)}{2t}\right)+\frac{\sigma_t^2}{2t}x_t
\end{aligned}
\label{eq:28}
\end{equation}

Thus:
\begin{equation}
\mu_{\theta}=x_t-\Delta t\cdot \text{drift}
\label{eq:29}
\end{equation}

Taking the gradient with respect to $\theta$ (noting that $x_t$ is constant):
\begin{equation}
\nabla_{\theta}\mu_{\theta}=-\Delta t\cdot\nabla_{\theta}\text{drift}=-\Delta t\cdot\left(1+\frac{\sigma_t^2(1-t)}{2t}\right)\nabla_{\theta}v_{\theta}
\label{eq:30}
\end{equation}

Substituting into $\nabla_{\theta}\log p_{\theta}$:
\begin{equation}
\begin{aligned}
\nabla_{\theta}\log p_{\theta} &= \frac{1}{\sigma_t\sqrt{\Delta t}} \epsilon\cdot\left[-\Delta t\cdot\left(1+\frac{\sigma_t^2(1-t)}{2t}\right)\nabla_{\theta}v_{\theta}\right] \\
&= -\frac{\sqrt{\Delta t}}{\sigma_t}\left(1+\frac{\sigma_t^2(1-t)}{2t}\right)\epsilon\cdot\nabla_{\theta}v_{\theta}
\end{aligned}
\label{eq:31}
\end{equation}

Therefore, the gradient of the policy loss is:
\begin{equation}
\nabla_{\theta}\mathcal{L}_{\text{policy}}(\theta)=\hat{A}_t^i r_t^i(\theta)\cdot\left[-\frac{\sqrt{\Delta t}}{\sigma_t}\left(1+\frac{\sigma_t^2(1-t)}{2t}\right)\epsilon\cdot\nabla_{\theta}v_{\theta}\right]
\label{eq:32}
\end{equation}

Substituting $a=1$ and $\sigma_t=\sqrt{\frac{t}{1-t}}$ (so $\sigma_t^2=\frac{t}{1-t}$). The coefficient term is computed as:
\begin{equation}
1+\frac{\sigma_t^2(1-t)}{2t}=1+\frac{\frac{t}{1-t}\cdot(1-t)}{2t}=1+\frac{1}{2}=\frac{3}{2}
\label{eq:33}
\end{equation}

And the scaling term:

\begin{equation}
\frac{\sqrt{\Delta t}}{\sigma_t}=\frac{\sqrt{\Delta t}}{\sqrt{\frac{t}{1-t}}}=\sqrt{\Delta t}\cdot\sqrt{\frac{1-t}{t}}=\sqrt{\frac{\Delta t(1-t)}{t}}
\label{eq:34}
\end{equation}

Substituting these simplifications gives the final policy gradient expression:
\begin{equation}
\nabla_{\theta}\mathcal{L}_{\text{policy}}(\theta)=-\frac{3}{2}\hat{A}_t^i\sqrt{\frac{\Delta t(1-t)}{t}}\epsilon\cdot\nabla_{\theta}v_{\theta}
\label{eq:35}
\end{equation}

A reweighting coefficient is introduced, defined as:
\begin{equation}
\lambda_{\text{policy}}(t,\Delta t)=\kappa(t,\Delta t)^{-1}=\sqrt{\frac{t}{\Delta t(1-t)}}
\label{eq:36}
\end{equation}

The reweighted policy loss becomes:
\begin{equation}
\mathcal{L}_{\text{policy, reweighted}}(\theta)=\lambda_{\text{policy}}(t,\Delta t)\cdot\mathcal{L}_{\text{policy}}(\theta)
\label{eq:37}
\end{equation}

This yields the modified gradient:
\begin{equation}
\nabla_{\theta}\mathcal{L}_{\text{policy, reweighted}}(\theta)=-\frac{3}{2}\hat{A}_t^i\cdot\epsilon\cdot\nabla_{\theta}v_{\theta}
\label{eq:38}
\end{equation}

Similarly, the gradient of the KL divergence term can be derived as:
\begin{equation}
\nabla_{\theta}D_{\text{KL}}(\theta)=\Delta t\cdot\frac{9}{4}\cdot\frac{1-t}{t}\cdot\left(v_{\theta}-v_{\text{ref}}\right)\cdot\nabla_{\theta}v_{\theta}
\label{eq:39}
\end{equation}

This expression reveals that the KL loss gradient suffers from the same scaling issues as the policy loss gradient. To address this, a KL reweighting coefficient is also introduced:
\begin{equation}
\lambda_{\text{KL}}(t,\Delta t)=k_{\text{KL}}(t,\Delta t)^{-1}=\frac{t}{\Delta t(1-t)}
\label{eq:40}
\end{equation}

The reweighted KL loss becomes:
\begin{equation}
\mathcal{L}_{\text{KL, reweighted}}(\theta)=\lambda_{\text{KL}}(t,\Delta t)\cdot D_{\text{KL}}(\theta)
\label{eq:41}
\end{equation}

yielding the simplified gradient:
\begin{equation}
\nabla_{\theta}\mathcal{L}_{\text{KL, reweighted}}(\theta)=\frac{9}{4}\cdot\left(v_{\theta}-v_{\text{ref}}\right)\cdot\nabla_{\theta}v_{\theta}
\label{eq:42}
\end{equation}

Based on the reweighting coefficients for the policy loss and KL loss, the revised GRPO objective function is as follows:
\begin{multline}
\mathcal{J}_{\text{GRPO}}(\theta)=\mathbb{E}_{
\substack{c\sim\mathcal{C},\,t'\sim\mathcal{U}(0,T'-1),\\
\{\boldsymbol{x}^i\}_{i=1}^{G}\sim\pi_{\theta_{\text{old}}}(\cdot|c,t')}}
\left[
\frac{1}{G}\sum_{i=1}^{G}
\Bigg(
\lambda_{\text{policy}}\left(\frac{t'}{T},\Delta\frac{t'}{T}\right)\cdot\mathcal{L}_{\text{policy}}(\theta)\right.\\
\left.-\beta\lambda_{\text{KL}}\left(\frac{t'}{T},\Delta\frac{t'}{T}\right)\cdot D_{\text{KL}}\left(\pi_{\theta}\|\pi_{\text{ref}}\right)
\Bigg)
\right]
\label{eq:43}
\end{multline}

\section{Reward Models and Reward Composition}\label{app:rewards}

For video-tower optimization, we use VideoAlign~\citep{liu2026improving}, CLIP~\citep{radford2021learning}, and DeSync~\citep{iashin2024synchformer} to assess video quality, video--prompt alignment, and audio--video synchronization, respectively. For audio-tower optimization, we use Audiobox Aesthetics~\citep{tjandra2025meta}, CLAP~\citep{wu2023large}, and DeSync to assess audio quality, audio--prompt alignment, and audio--video synchronization, respectively.

Following GDPO, we first compute group-wise normalized advantages for each individual metric, then sum them to obtain a cumulative score per sample, and finally perform a second group-wise normalization to obtain the final advantage \( \hat{A}_t^i \).

Specifically, when training the video tower with audio anchored, the cumulative score for video sample \( i \) is:

\[
S_{\text{video}}^i = \text{norm}(\text{VQ}^i) + \text{norm}(\text{CLIP}^i) + \text{norm}(\text{-Desync}^i),
\]

where \( \text{norm}(r^i) = \frac{r^i - \text{mean}(\{r^i\}_{i=1}^G)}{\text{std}(\{r^i\}_{i=1}^G)} \) denotes group-wise standardization. \textbf{Note that for the DeSync metric, the lower the better.} The final advantage is:

\[
\hat{A}_{\text{video}}^i = \text{norm}\big(\{S_{\text{video}}^i\}_{i=1}^G\big).
\]

Symmetrically, when training the audio tower with video anchored as \( \tau_v \), the cumulative score for audio sample \( i \) is:

\[
S_{\text{audio}}^i = \text{norm}(\text{AQ}^i) + \text{norm}(\text{CLAP}^i) + \text{norm}(\text{-Desync}^i),
\]

with the final advantage:

\[
\hat{A}_{\text{audio}}^i = \text{norm}\big(\{S_{\text{audio}}^i\}_{i=1}^G\big).
\]

In practice, however, we observe that when training the audio tower, the model is prone to reward hacking: it easily finds a shortcut that sacrifices the CLAP score while substantially boosting AQ and Desync. This behavior severely undermines audio-prompt semantic alignment and is unacceptable.

Inspired by GDPO, we introduce a \textbf{semantic alignment guardrail} for CLAP. Let the group-wise average CLAP score be \( \bar{r}_{\text{CLAP}} = \frac{1}{G}\sum_{i=1}^G \text{CLAP}^i \), with a predefined threshold \( \eta_{\text{CLAP}} \). Only when \( \bar{r}_{\text{CLAP}} \geq \eta_{\text{CLAP}} \) are AQ and Desync incorporated into the advantage; otherwise, the CLAP advantage serves directly as the final advantage. Formally:

\[
\hat{A}_{\text{audio}}^i =
\begin{cases}
\text{norm}\big(\{S_{\text{audio}}^i\}_{i=1}^G\big), & \bar{r}_{\text{CLAP}} \geq \eta_{\text{CLAP}}, \\[6pt]
\text{norm}\big(\{\text{CLAP}^i\}_{i=1}^G\big), & \bar{r}_{\text{CLAP}} < \eta_{\text{CLAP}}.
\end{cases}
\]

When the group-average CLAP falls below the threshold, the model receives learning signals exclusively from CLAP, guiding it to prioritize audio-prompt semantic alignment before attending to other metrics.

\section{More Qualitative Results}\label{app:vis}

\begin{figure}[h]
    \centering
    \includegraphics[width=1.0\linewidth]
    {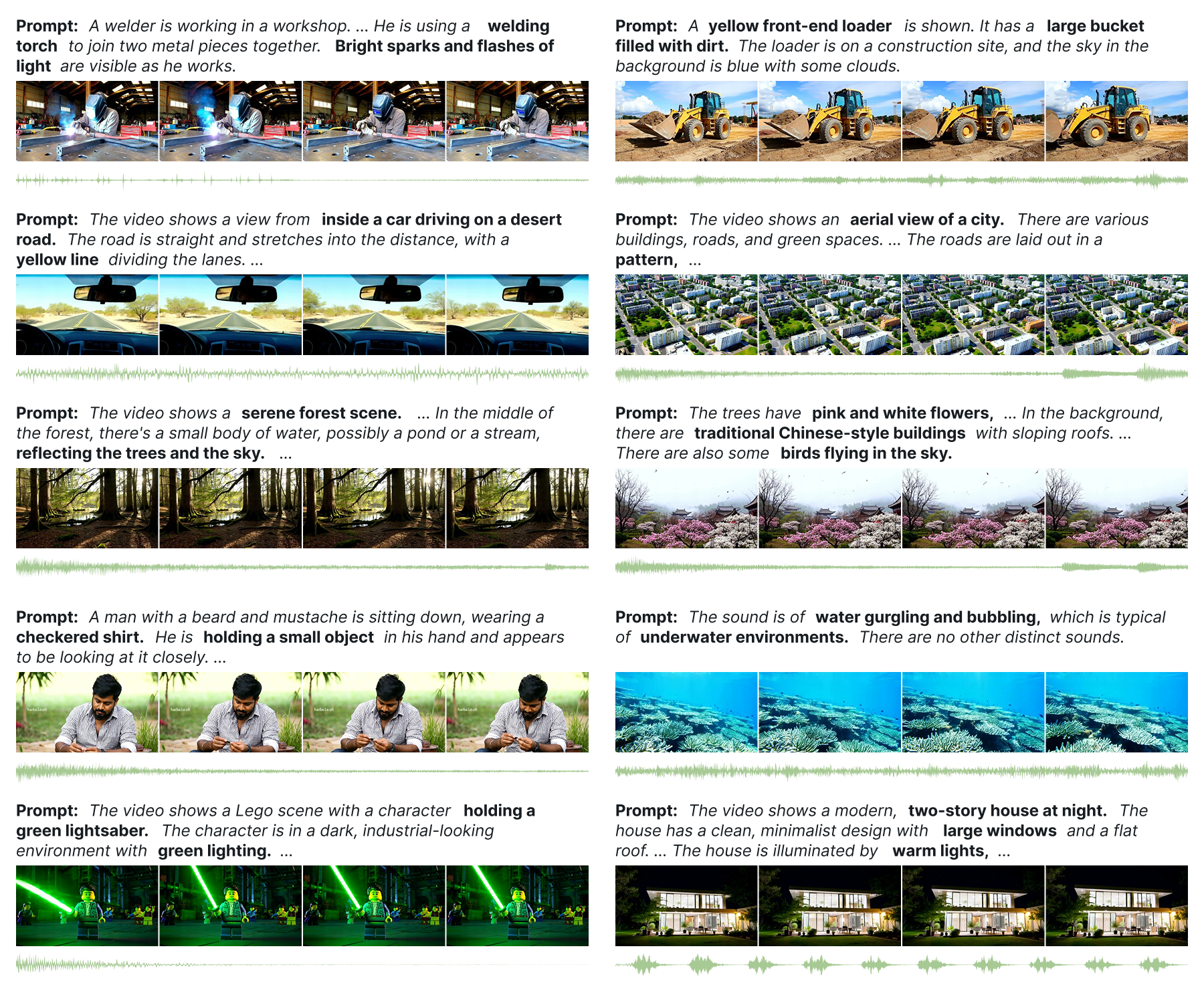}
    \vspace{-3mm}
    \caption{More qualitative results of AV-GRPO~(full).}
    \label{fig:appendix_vis}
    \vspace{-3mm}
\end{figure}

\end{document}